%% file: Main_ALL.tex
\documentclass[journal]{IEEEtran}

\usepackage{amsthm,amsmath,amsfonts,bbm}
\usepackage{array}
\usepackage{textcomp}
\usepackage{stfloats}
\usepackage{url}
\usepackage{verbatim}

\usepackage{cite}

\usepackage{amssymb}
\usepackage{bm}
\usepackage{booktabs}
\usepackage{lipsum}
\usepackage{enumerate}

\usepackage{graphicx}
\usepackage{subfigure}

\usepackage[ruled,linesnumbered]{algorithm2e}
\usepackage{pifont}

\usepackage[colorlinks,
linkcolor=blue,
anchorcolor=blue,
citecolor=blue]{hyperref}
\usepackage{color}

\allowdisplaybreaks[4]

\newtheorem{lemma}{Lemma}

\newtheorem{theorem}{Theorem}

\newtheorem{assumption}{Assumption}
\newtheorem{corollary}{Corollary}

\begin{document}
	
\title{Online Conformal Prediction with Corrupted Feedback}

\author{Bowen Wang,~\IEEEmembership{Graduate Student Member,~IEEE},
Matteo Zecchin,~\IEEEmembership{Member,~IEEE},\\
and Osvaldo Simeone,~\IEEEmembership{Fellow,~IEEE} 
\thanks{The work of O. Simeone was supported by the Open Fellowships of the Engineering and Physical Sciences Research Council (EPSRC) under Grant EP/W024101/1, in part by the EPSRC project under Grant EP/X011852/1, and in part by the European Research Council (ERC) under Grant No. 101198347.}
\thanks{Bowen Wang is with the Department of Engineering, King's College London, WC2R 2LS London, U.K. (e-mail: bowen.wang@kcl.ac.uk).}
\thanks{Matteo Zecchin is with the Communication Systems Department, EURECOM, 06410 Sophia Antipolis, France (e-mail: zecchin@eurecom.fr).}
\thanks{Osvaldo Simeone is with the Institute for Intelligent Networked Systems, Northeastern University London, E1W 1LP London, U.K. (e-mail: o.simeone@northeastern.edu).}
}
	
\maketitle

\begin{abstract}
	Modern artificial intelligence systems require calibrated uncertainty estimates that remain reliable in sequential and non-stationary environments. 
	Online conformal prediction (OCP) addresses this challenge through adaptively updated prediction sets that provide deterministic long-run miscoverage guarantees. 
	These guarantees, however, hinge on the assumption of perfect feedback about the coverage of past prediction sets.
	In practice, the observed miscoverage indicator may be corrupted by noise, communication failures, or adversarial manipulation, which can severely degrade OCP's calibration guarantees. 
	In this paper, we study OCP under corrupted feedback. 
	We first model feedback corruption as an arbitrary binary flip sequence, and analyze how feedback corruption affects and degrades the miscoverage performance of standard OCP. 
	We then propose two robust schemes: robust OCP via filtering, which leverages the structural properties of the predicted threshold to filter corrupted feedback, and robust OCP via active compensation, which incorporates an active compensation mechanism to mitigate the effect of corrupted feedback.
	For both methods, we establish explicit miscoverage guarantees, which are further specialized for an independent stochastic flip model and for an arbitrary error model with memory bounds. 
	Experiments on real-world datasets validate the proposed approach, showing markedly improved calibration and significantly smaller prediction sets compared with baseline OCP methods under corrupted feedback.
\end{abstract}

\section{Introduction}

\subsection{Context and Motivation}

The reliable deployment of modern artificial intelligence (AI) systems  increasingly depends not only on predictive accuracy but also on the ability to rigorously quantify uncertainty \cite{gawlikowski2023survey}. 
In safety-critical and decision-oriented applications \cite{abdar2021review}, a model must be able to assess when its predictions can be trusted and when uncertainty should be explicitly communicated to downstream decision-makers or controllers \cite{Bingcong2019TSP,Chlaily2023TSP,Nir2025TSP,Dua2026ICASSP}. 
This requirement arises across a broad range of domains \cite{simeone2025uncertainty}. 
In large language model (LLM)-based systems \cite{shorinwa2025survey}, calibrated uncertainty estimates can inform selective abstention, tool invocation, and human-in-the-loop interaction. 
In signal processing and engineering applications, such as robotics \cite{ankile2024juicer}, industrial monitoring \cite{Ji2024TSP}, wireless networks \cite{zecchin2026prediction} and distributed systems \cite{zhu2025conformal}, calibrated confidence measures are essential for robust decision-making under distribution shift, partial observability, and non-stationary conditions \cite{simeone2025conformal}.

The need for uncertainty quantification of AI models becomes particularly acute in sequential settings, where data arrive over time and operating conditions may evolve continually \cite{simeone2025conformal}. 
In such environments, classical offline calibration methods \cite{angelopoulos2020uncertainty} are often inadequate, as they rely on exchangeability or stationarity assumptions and offer only \textit{statistical} coverage guarantees.
By contrast, many real-world systems demand online calibration mechanisms that can adapt to changing conditions while providing \textit{deterministic}, worst-case guarantees over finite horizons \cite{chernozhukov2018exact}.

\textit{Online conformal prediction} (OCP) \cite{gibbs2021adaptive,gibbs2024conformal,bhatnagar2023improved,zecchin2024localized,wu2025error,xu2023conformal} has recently emerged as a principled framework for addressing this challenge. 
By updating the prediction threshold sequentially based on observed coverage errors, OCP can control the long-run miscoverage rate in an online and distribution-agnostic manner \cite{gibbs2021adaptive,gibbs2024conformal,bhatnagar2023improved,zecchin2024localized,wu2025error,wang2025mirror}. 
Its appeal lies in the fact that the calibration mechanism is adaptive, computationally lightweight, and amenable to deterministic finite-horizon analysis. 
However, this guarantee relies critically on an assumption that is often left implicit: the feedback about the coverage of past prediction sets must be correct \cite{gibbs2021adaptive,gibbs2024conformal,bhatnagar2023improved,zecchin2024localized,wu2025error}.

This assumption can be fragile in practice. 
In deployed systems, coverage feedback may be corrupted by communication failures \cite{Fishel2025TSP}, annotation noise, delayed or inconsistent supervision, sensor malfunctions, or adversarial manipulation \cite{zecchin2026prediction,Fishel2025TSP,simeone2025conformal}. 
In LLM-assisted pipelines, for example, whether an output should be deemed ``covered'' may itself depend on noisy automatic evaluation or downstream agent responses. 
In cyber-physical and communication-based systems, feedback may be distorted by transmission errors \cite{zecchin2026prediction}, by conflicts among multiple decision-making modules, or by malicious manipulations that may drive a  control process away from its intended operating point \cite{alimohammadi2024kpi}.
Corrupted feedback is problematic for OCP, since when the coverage feedback contradicts the true coverage indicator, the OCP update may cause the threshold to drift systematically away from the value required for correct calibration. 
As a result, the deterministic guarantees of OCP may break down even when the corruption level is moderate.

Motivated by this gap, in this paper we study OCP under corrupted feedback. 
Our goal is both to understand how feedback corruption affects the calibration dynamics of OCP and to develop robust online conformal algorithms that retain strong miscoverage guarantees.

\subsection{Related Work}

\subsubsection{Conformal Prediction with Label Noise}

The effect of corrupted labels on \textit{conformal prediction} (CP) has been studied primarily in offline settings, where noisy labels are used to construct the threshold \cite{einbinder2024label,feldman2025conformal,xi2025exploring,Chi2026ICASSP}. 
In \cite{einbinder2024label}, prediction sets are built directly from corrupted calibration datasets under independent and identically distributed (i.i.d.) assumptions, and the resulting miscoverage analysis shows that label corruption induces conservative behavior, leading to over-coverage. 
This framework is extended in \cite{feldman2025conformal} beyond the i.i.d. setting through weighted CP, which is used to account for distribution shift.
However, these methods are not designed for sequential settings and do not provide deterministic pathwise guarantees. 
More closely related to the present work, reference \cite{xi2025exploring} proposes a robust OCP method under uniform label noise by constructing an unbiased estimator of the threshold-update gradient. 
Nevertheless, the method requires prior knowledge of the noise level and does not accommodate more general feedback corruption models.

\begin{figure*}[t]
	\centering
	\subfigure[]{\label{Fig_0a}
		\includegraphics[width=1\textwidth]{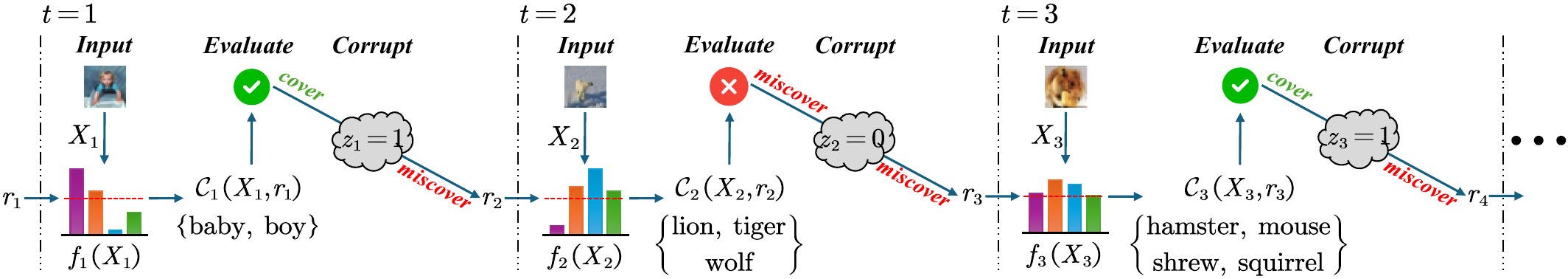}
	}
	\subfigure[]{\label{Fig_0b}
		\includegraphics[width=0.4\textwidth]{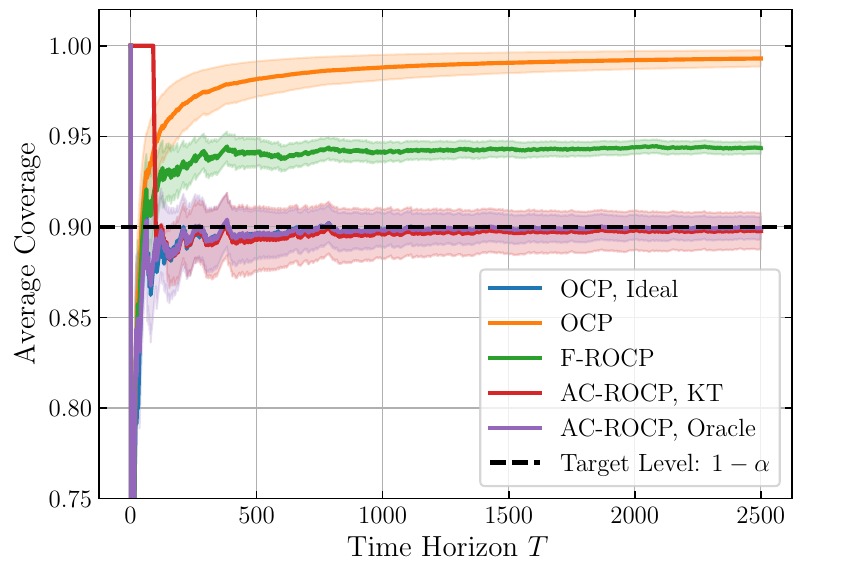}
		\includegraphics[width=0.4\textwidth]{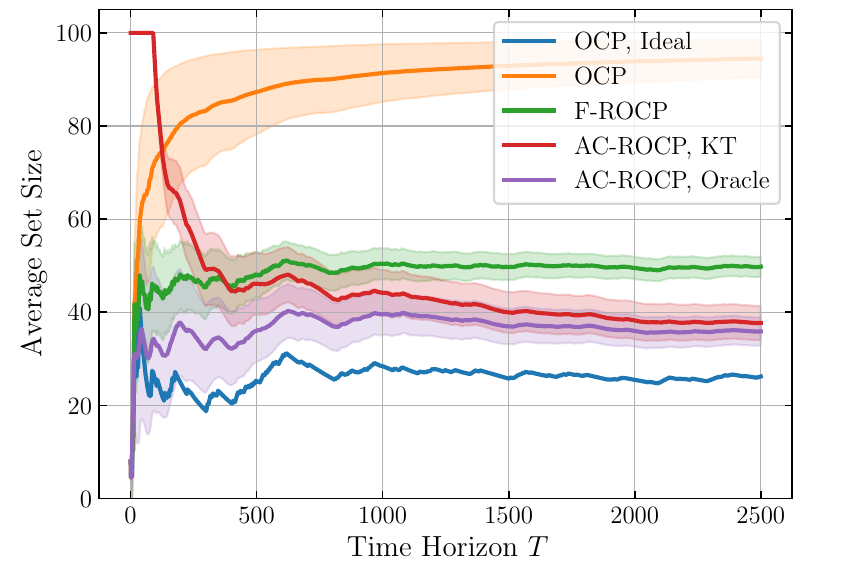}
	}
	\caption{(a) Operation of online conformal prediction (OCP) \cite{gibbs2021adaptive,gibbs2024conformal,bhatnagar2023improved,zecchin2024localized,wu2025error,xu2023conformal} under corrupted feedback. At each time step $t$, given an input $X_t$, OCP maps the output of the prediction model $f(X_t)$ to the prediction set $\mathcal{C}_t(X_t; r_t)$, whose size is governed by the threshold $r_t$. 
	Evaluation against the true label $Y_t$ yields the miscoverage indicator $e_t \in \{ 0 , 1 \}$, but the learner only observes the corrupted feedback $\bar{e}_t \in \{ 0 , 1 \}$ with feedback error indicator $z_t = \mathbbm{1} \{ e_t \ne \bar{e}_t\} $.
	The threshold is then updated based on the feedback $\bar{e}_t$ to $r_{t+1}$.
	(b) Averaged coverage (left) and set size (right) of the considered methods over the time horizon $T$ with target miscoverage rate $\alpha = 0.1$ for the CIFAR-100 based classification task.
	Results are averaged over $10^4$ independent trials, with the shaded region denoting a two-sided interval with half-width equal to one standard deviation.
	Conventional OCP fails to meet the coverage requirement, significantly exceeding the target coverage level $1-\alpha = 0.9$. In contrast, the proposed methods, namely robust OCP via filtering (F-ROCP) and robust OCP via active compensation (AC-ROCP), achieve the target coverage accurately, while producing significantly smaller and hence more informative prediction sets than the conventional OCP.}
	\label{Fig_0}
\end{figure*}

\subsubsection{Online Convex Optimization with Corrupted Feedback}

Since OCP is typically implemented via \textit{online subgradient descent} (OGD) \cite{zinkevich2003online}, the problem of OCP with corrupted feedback can also be viewed as a special instance of \textit{online convex optimization} (OCO) \cite{orabona2019modern} with corrupted feedback.
Existing work in this area can be broadly divided into two main lines of research: one focuses on stochastic corruption models, while the other studies adversarial corruption in unconstrained OCO.
In the first line of work, the corrupted gradients are assumed to be generated by a stochastic process \cite{flaxman2004online,zhang2022parameter,van2019user,Cao2019TAC}, and the resulting regret bounds scale with the noise level. 
In the second line of research, an adversary is allowed to corrupt the gradients arbitrarily, and the corresponding regret guarantees depend on the cumulative corruption level \cite{van2021robust,zhang2025unconstrained}.
While these results provide useful insights into optimization under corrupted gradient information, their objective is to ensure sublinear regret. 
This is fundamentally different from the goal of OCP, where the key requirement is calibration rather than optimization performance. 
In particular, regret guarantees do not directly translate into miscoverage guarantees \cite{angelopoulos2025gradient}, and hence are not sufficient to characterize the calibration behavior of OCP under corrupted feedback.

\subsection{Contributions}
Motivated by the need to develop robust OCP mechanisms under corrupted feedback, the main contributions of this paper are summarized as follows:

\begin{itemize}
	\item As shown in Fig. \ref{Fig_0a}, we introduce a general feedback corruption model in which the true miscoverage indicator is flipped according to an arbitrary binary corruption process. 
	Under this model, we analyze how corrupted feedback affects both the threshold dynamics and the miscoverage performance of standard OCP (see Fig. \ref{Fig_0b}).

	\item Building on this analysis, we propose two robust OCP schemes: \textit{robust OCP via filtering} (F-ROCP), which requires no side information and exploits only the intrinsic structural properties of OCP; and \textit{robust OCP via active compensation} (AC-ROCP), which incorporates an active compensation mechanism to mitigate the effect of corrupted feedback. 
	For both methods, we establish explicit miscoverage guarantees.

	\item We specialize the proposed framework to two representative corruption models: an independent stochastic flip model and an arbitrary error model with memory bounds. 
	For both settings, we instantiate F-ROCP and AC-ROCP and derive explicit miscoverage bounds that characterize the dependence on the corruption level and, for AC-ROCP, on the quality of the compensation mechanism.

	\item We validate the proposed methods through extensive experiments on real-world datasets under both corruption models. 
	The results show that the proposed schemes substantially improve the calibration performance of OCP under corrupted feedback, while also achieving significantly smaller and hence more informative prediction sets than the baseline methods (see Fig. \ref{Fig_0b}).
\end{itemize}

The rest of this paper is organized as follows. 
In Sec. \ref{Sec_2}, we introduce the online conformal prediction setting and the feedback corruption models considered in this work. 
Sec. \ref{Sec_3} analyzes the impact of corrupted feedback on standard OCP and shows empirical results. 
Sec. \ref{Sec_4} presents the F-ROCP method, while Sec. \ref{Sec_5} introduces the AC-ROCP. 
Sec. \ref{Sec_6} and \ref{Sec_7} study two representative corruption models, and compare different approaches.
Sec. \ref{Sec_8} reports experimental results, and Sec. \ref{Sec_9} concludes the paper.

\section{Problem Definition}\label{Sec_2}
In this section, we describe relevant background material, and introduce the corrupted feedback model.

\subsection{Online Conformal Prediction}

Following the standard OCP framework \cite{gibbs2021adaptive,gibbs2024conformal,bhatnagar2023improved,zecchin2024localized,wu2025error,xu2023conformal}, as illustrated in Fig. \ref{Fig_0a}, a sequence of data points $\{(X_t, Y_t)\}_{t=1}^{T}$ arrives sequentially, where $X_t \in \mathcal{X}$ denotes the input feature vector, $Y_t \in \mathcal{Y}$ is the corresponding true label, and $T$ denotes the time horizon. 
No distributional assumptions are imposed on the sequence: the sequence $\{(X_t, Y_t)\}_{t=1}^{T}$ may be chosen adversarially or exhibit arbitrary temporal dependence.

At each time step $t$, a pre-trained base model $f \colon \mathcal{X} \to \mathcal{Y}$ produces a point prediction for input $X_t$. 
Rather than relying on this point prediction alone, OCP quantifies the uncertainty in the prediction by leveraging the model's output $f(X_t)$ to produce a \textit{prediction set} $\mathcal{C}_t(X_t) \subseteq \mathcal{Y}$.
Formally, for a user-specified miscoverage rate $\alpha \in [0,1]$, the goal of OCP is to ensure that the long-run empirical miscoverage converges to a target value $\alpha \in [0 , 1]$, i.e.,
\begin{equation}\label{ex_long_cov}
	\lim_{T \to \infty} \frac{1}{T} \sum_{t=1}^{T} \mathbbm{1} \{ Y_t \notin \mathcal{C}_t(X_t) \} = \alpha .
\end{equation}

To this end, OCP adopts a non-conformity score function $S(x, y): \mathcal{X} \times \mathcal{Y} \to \mathbb{R}$ to measure the discrepancy of the pair $(x, y)$ relative to the model prediction $f(x)$.
For example, in classification tasks, a common choice is $S(x, y) = 1 - [f(x)]_y$ \cite{angelopoulos2024online}, where $[f(x)]_y$ is the predicted probability of class $y$ for input $x$, while in regression tasks a common choice is the absolute prediction error $S(x, y) = |y - f(x)|$ \cite{gibbs2021adaptive}.
The prediction set $\mathcal{C}_t(X_t; r_t)$ is defined as the set of labels whose non-conformity scores do not exceed a \textit{prediction threshold} $r_t$, i.e.,
\begin{equation}\label{Eq_Pre_Set}
	\mathcal{C}_t(X_t; r_t) = \left\{ y \in \mathcal{Y} : S(X_t, y) \le r_t \right\} .
\end{equation}

The environment then reveals whether the true label $Y_t$ is covered by the prediction set. The resulting \textit{miscoverage indicator} is
\begin{equation}\label{eq_err}
	e_t = \mathbbm{1}\{ Y_t \notin \mathcal{C}_t(X_t; r_t) \} = \mathbbm{1}\{ r_t < s_t \} ,
\end{equation}
where $s_t = S(X_t, Y_t)$ is the true non-conformity score of the true label at time $t$.
This feedback is used to update the threshold $r_t$ for the next round $t+1$.

Specifically, OCP can be viewed as an application of \textit{online gradient descent} (OGD) to the problem of minimizing the cumulative quantile loss at level $1 - \alpha$. 
In particular, define the $(1 - \alpha)$-th quantile loss as
\begin{equation}
	\ell_{1 - \alpha} \left( r , s_t \right) = ( \alpha - \mathbbm{1}\{ r < s_t \} ) ( r - s_t ).
\end{equation}
After observing the miscoverage indicator $e_t$ in \eqref{eq_err}, OCP performs OGD to update the threshold $r_t$ as \cite{gibbs2021adaptive,gibbs2024conformal,bhatnagar2023improved,zecchin2024localized,wu2025error,xu2023conformal}
\begin{equation}\label{eq_clean_OCP}
	\begin{aligned}
		r_{t+1} & = r_t - \eta g_t \\
		& = r_t - \eta \left( \alpha - e_t \right) ,
	\end{aligned}
\end{equation}
where $\eta > 0$ is the \textit{learning rate}, and 
\begin{equation}\label{eq_clean_grad}
	g_t = \nabla_r \ell_{1 - \alpha} \left( r_t , s_t \right) = \alpha - e_t .
\end{equation}
is the \textit{subgradient} of the quantile loss evaluated at the threshold $r_t$.

The update \eqref{eq_clean_OCP} admits a simple and intuitive interpretation. When a miscoverage occurs, i.e., $e_t = 1$, the gradient $g_t$ is negative, and the threshold $r_t$ increases by $\eta(1 - \alpha)$, thereby enlarging future prediction sets.
Conversely, when the true label is covered, i.e., $e_t = 0$, the threshold decreases by $\eta \alpha$, shrinking future prediction sets and improving efficiency.

Before presenting the formal miscoverage guarantees, we state an assumption used throughout the subsequent analysis.
\begin{assumption}[\!\!\cite{gibbs2021adaptive,hu2026distributioninformed,angelopoulos2024online}]\label{Assumption_1}
	The score function $S(x, y)$ is bounded within the interval $[0, B]$ for some constant $0 < B < \infty$, i.e., we have the inclusion $s_t \in [0, B]$ for all $t = 1 , 2 , \ldots, T$.
\end{assumption}

The empirical miscoverage performance of OCP is summarized by the following lemma provided in \cite{gibbs2021adaptive}.
\begin{lemma}[\!\!\cite{gibbs2021adaptive}]\label{Lem_OCP_Clean}
	Under Assumption~\ref{Assumption_1}, the empirical miscoverage rate of OCP with ideal feedback satisfies the inequality 
	\begin{equation}\label{eq_mis_ideal_ocp}
		\mathrm{MisCov}(T) = \left| \alpha - \frac{1}{T} \sum_{t=1}^{T} \mathbbm{1}\{ Y_t \notin \mathcal{C}_t(X_t; r_t) \}  \right| \le \frac{B + \eta}{\eta T} .
	\end{equation}
\end{lemma}
From Lemma~\ref{Lem_OCP_Clean}, it is evident that the miscoverage rate $\mathrm{MisCov}(T)$ vanishes as $T$ grows.
Consequently, the long-run average miscoverage converges to the target level $\alpha$.

\subsection{Corrupted Feedback}\label{Sec_II_B}
The performance guarantees of OCP, summarized in Lemma~\ref{Lem_OCP_Clean}, critically rely on access to the exact feedback signal $e_t$, which is required to construct the gradient $g_t$ in \eqref{eq_clean_grad} for the OGD based update \eqref{eq_clean_OCP}. 
In practice, however, this feedback may be corrupted by noise, communication errors, or adversarial manipulation, potentially leading to severe degradation in coverage performance.
Accordingly, in this paper, we assume that the true miscoverage indicator $e_t \in \{ 0 , 1  \}$ is not directly observed. Instead, the learner receives a corrupted version $\bar{e}_t \in \{ 0 , 1  \}$, which may differ from $e_t$.
We model the \textit{feedback error indicator} as
\begin{equation}\label{eq_corrupted_err}
	z_t = \mathbbm{1} \{ \bar{e}_t \ne e_t \}  .
\end{equation}
Unless stated otherwise, we make no assumptions on the feedback error sequence $\{z_t\}_{t=1}^T$.
Under corrupted feedback \eqref{eq_corrupted_err}, the learner reconstructs the noisy gradient
\begin{equation}\label{eq_corrupted_grad}
	\overline{g}_t = \alpha - \bar{e}_t ,
\end{equation}
where the corrupted feedback $\bar{e}_t$ replaces the true error $e_t$ in \eqref{eq_clean_grad}.

We denote the set of time indices in which a feedback error occurs as
\begin{equation}\label{Eq_11}
	\begin{aligned}
		\bar{\mathcal{G}}^T & = \left\{  t \in \{1, 2, \cdots, T\} : z_t = 1 \right\} \\
		& = \left\{  t \in \{1, 2, \cdots, T\} : g_t \neq \bar{g}_t \right\} ,
	\end{aligned}
\end{equation}
and a natural measure of feedback quality is the \textit{cumulative corruption level}
\begin{equation}
	G^T = |\bar{\mathcal{G}}^T| = \sum_{t=1}^{T} z_t = \sum_{t=1}^T \mathbbm{1}\{ g_t \neq \bar{g}_t \}  .
\end{equation}

\section{Impact of Corrupted Feedback on Online Conformal Prediction}\label{Sec_3}

In this section, we evaluate the impact of corrupted feedback, as in \eqref{eq_corrupted_err}-\eqref{eq_corrupted_grad}, on the performance of conventional OCP. 
We start with a numerical example and then provide a formal analysis of the miscoverage rate of OCP under corrupted feedback.

\subsection{Numerical Example}

Referring to Sec. \ref{Sec_8} for full details on the setting, we consider here an image classification task based on CIFAR-100 dataset, where the objective is to construct prediction sets that contain the true class label with target miscoverage level $\alpha = 0.1$.
Under corrupted feedback, the OCP update \eqref{eq_clean_OCP} becomes
\begin{equation}\label{OCP_Noisy}
	r_{t+1} = r_t - \eta \bar{g}_t ,
\end{equation}
that is, the corrupted gradient in \eqref{eq_corrupted_grad} is directly used to update the threshold.
Fig. \ref{Fig_1} reports the average coverage rate and average set size under the assumption that the feedback error indicators $z_t$ are i.i.d. Bernoulli with parameter $p = \Pr \{ z_t = 1\}$.

In Fig. \ref{Fig_1}, as the corruption probability $p$ increases, the empirical coverage $\sum_{t=1}^T \mathbbm{1} \{ Y_t \in \mathcal{C}_t(X_t) \} / T$ increasingly deviates from the target level $1 - \alpha = 0.9$.
In fact, as seen in Fig. \ref{Fig_1_b}, the OCP update \eqref{OCP_Noisy} leads to excessively conservative predictions with increasingly larger prediction sets as the error probability $p$ grows. This numerical example illustrates the general fact that OCP fails to meet the target coverage rate under corrupted feedback.

\begin{figure}[t]
	\centering
	\subfigure[]{\label{Fig_1_a}
		\includegraphics[width=0.4\textwidth]{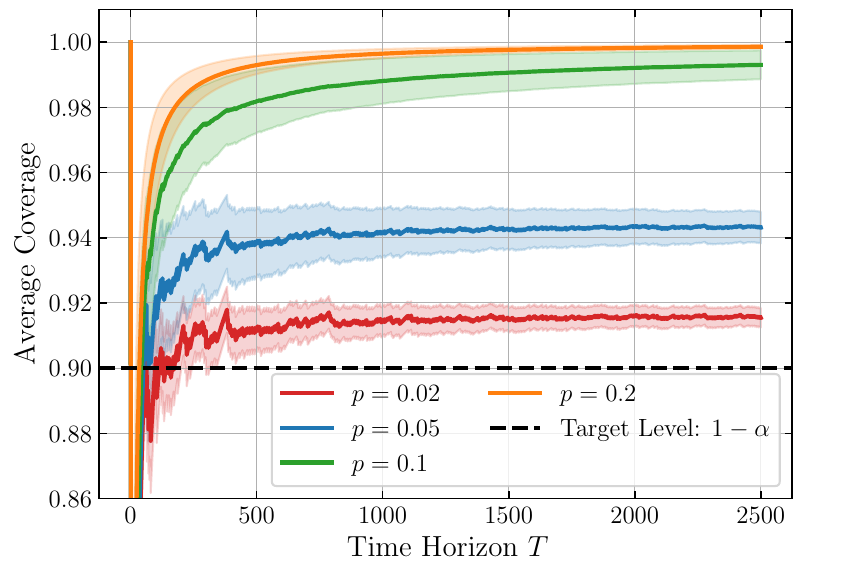}
	}
	\subfigure[]{\label{Fig_1_b}
		\includegraphics[width=0.4\textwidth]{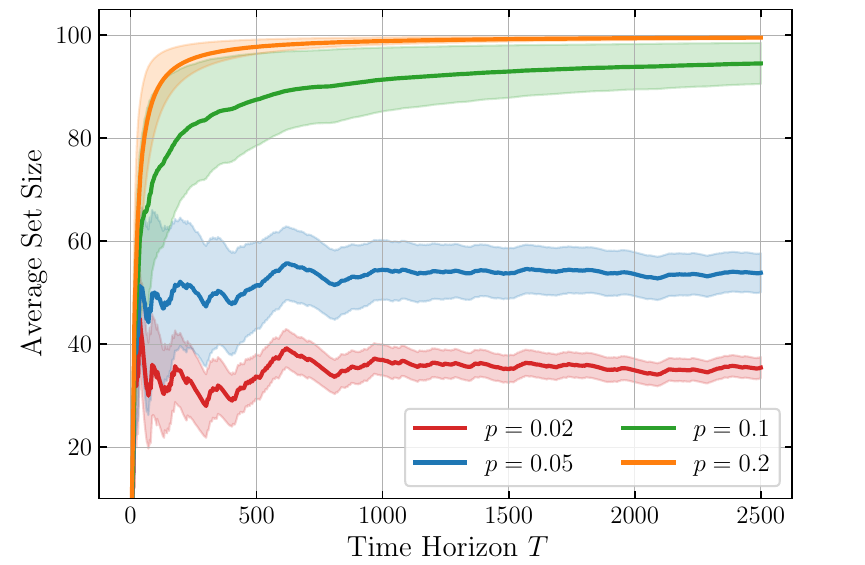}
	}
	\caption{Average coverage and average set size performance of OCP under corrupted feedback on the CIFAR-100 dataset with target miscoverage rate $\alpha = 0.1$ for the independent stochastic corruption model with different corruption probability $p$. Results are averaged over $10^4$ independent trials, with the shaded region denoting the standard deviation.}
	\label{Fig_1}
\end{figure}

\subsection{Miscoverage Analysis}

We now provide a theoretical characterization of the miscoverage behavior of OCP under corrupted feedback.

\begin{theorem}\label{Theo_1}
	Under Assumption~\ref{Assumption_1}, the empirical miscoverage rate of OCP under gradient corruption satisfies the upper bound
	\begin{equation}\label{New_2_15}
		\begin{aligned}
			\mathrm{MisCov}(T)
			\le & \frac{ B + \eta}{\eta T}
			+ \frac{ | G_{0\to1}^T - G_{1\to0}^T | }{T}  \\
			& + \frac{\max\{ \alpha G_{1\to0}^T , (1 - \alpha) G_{0\to1}^T \}}{T} ,
		\end{aligned}
	\end{equation}
	where $G_{0 \to 1}^T = | \{ t \in \bar{\mathcal{G}} : g_t = \alpha, \bar{g}_t = \alpha - 1 \} |$ and $G_{1 \to 0}^T = | \{ t \in \bar{\mathcal{G}} : g_t = \alpha - 1, \bar{g}_t = \alpha \} |$ denote the number of feedback errors that flip the gradient from $g_t = \alpha$ to $\bar{g}_t = \alpha - 1$ and from $g_t = \alpha - 1$ to $\bar{g}_t = \alpha$, respectively.
\end{theorem}
\begin{IEEEproof}
	Please see supplementary material Appendix \ref{App_Theo_1}.
\end{IEEEproof}

Unlike the ideal feedback case in Lemma \ref{Lem_OCP_Clean}, the miscoverage bound in Theorem \ref{Theo_1} contains three terms that capture different aspects of the impact of corrupted feedback on OCP:
\begin{itemize}
	\item The first term coincides with the bound in the ideal feedback case in \eqref{eq_mis_ideal_ocp}, which vanishes as the time horizon $T$ grows.

	\item The second term quantifies the \textit{directional imbalance} of the corruption. 
	Recall that OCP raises the threshold $r_t$ after a reported miscoverage and lowers it after a reported coverage. 
	When the feedback is ideal, these two types of updates balance out at the target rate $\alpha$. 
	Under corrupted feedback, however, a miscoverage is reported even when the prediction set includes the true label for $G_{0 \to 1}^T$ times, while coverage events are reported erroneously for $G_{1 \to 0}^T$ times.
	If the two types of feedback errors occur at the same rate, their effects cancel; but any imbalance between the respective numbers of errors $G_{0 \to 1}^T$ and $G_{1 \to 0}^T$ entails that the threshold is pushed up more often than down, or vice versa, relative to what the true outcomes would warrant in OCP, translating into a miscoverage bias.

	\item The third term captures the \textit{drift} in the threshold iterate $r_t$. 
	Under ideal feedback, the true gradient pulls the threshold $r_t$ back whenever it drifts. 
	Worst-case accumulation of these wrong-direction updates inflates the absolute value of threshold $|r_{T+1}|$, producing a drift. 
\end{itemize}

\section{Robust Online Conformal Prediction via Filtering}\label{Sec_4}
The analysis in the previous section reveals that blindly trusting the observed gradient breaks the self-stabilizing feedback loop of OCP. 
This motivates a natural remedy: filter the feedback before it is used in the update of OCP. 
Based on this intuition, in this section, we propose \emph{robust OCP via filtering} (F-ROCP), a simple yet effective scheme that exploits the intrinsic structure to reject corrupted feedback and restore the iterate bound.

\subsection{Protocol}
Assumption~\ref{Assumption_1} and prediction set definition in \eqref{Eq_Pre_Set} imply that, at any time step $t$, if the threshold leaves the interval $[0 , B)$, the true gradient can be inferred exactly (step \ding{172} and blue area in Fig. \ref{Fig_FROCP}).
Specifically, when the threshold exceeds the maximum score value, i.e., $r_t \ge B$, the prediction set \eqref{Eq_Pre_Set} contains all valid labels. 
Hence, the true miscoverage indicator must be $e_t = 0$, and the true gradient is therefore $g_t = \alpha$, regardless of the observed corrupted feedback. 
Likewise, when the threshold $r_t < 0$, the threshold is below the minimum feasible score level, so the true gradient must be $g_t = \alpha - 1$. 
Therefore, whenever the threshold $r_t$ lies outside the interval $[0 , B)$, which we refer to as the \textit{out-of-range} regime, the learner can recover the true gradient directly from the threshold value itself.

\begin{figure}[t]
	\centering
	\includegraphics[width=0.5\textwidth]{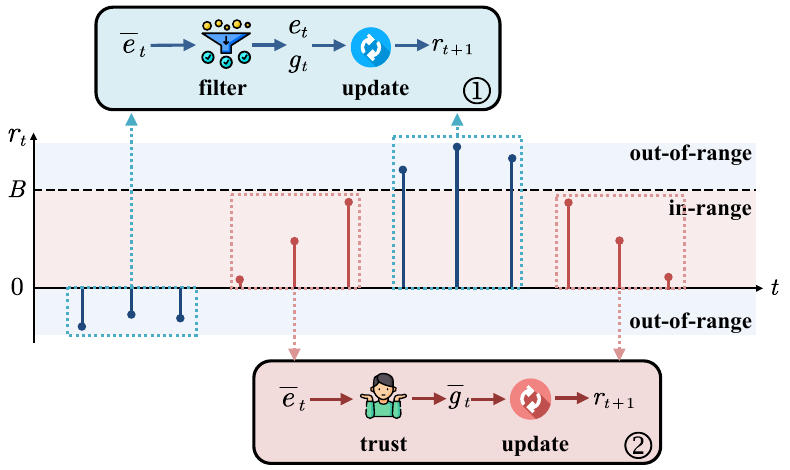}
	\caption{Illustration of F-ROCP:  
	When the threshold $r_t$ is out-of-range, the true gradient can be inferred exactly, and the correct OCP update is applied to  the threshold.
	When the threshold is in-range, the gradient information may be corrupted,  but F-ROCP directly uses it to update the threshold.
	}
	\label{Fig_FROCP}
\end{figure}

Motivated by this observation, we propose F-ROCP, a variant of OCP that filters the corrupted feedback based on the value of the threshold $r_t$. 
As summarized in Algorithm \ref{Alg_1} and Fig. \ref{Fig_FROCP}, when the threshold $r_t$ is out-of-range, i.e., $r_t \notin [0, B)$ (step \ding{173} and blue region in Fig. \ref{Fig_FROCP}), the algorithm rejects the corrupted feedback and instead infers the ideal gradient according to the equation
\begin{equation}\label{eq_infer_grad}
	g_t = \left\{ 
		\begin{array}{ll}
		\alpha , & \text{ if } r_t \ge B, \\
		\alpha - 1 , & \text{ if } r_t < 0 .
		\end{array}
	\right.
\end{equation}
The threshold is then updated using this inferred ideal gradient.
In contrast, when the threshold $r_t \in [0, B)$, which we refer to as the \textit{in-range} regime, i.e., the red region in Fig. \ref{Fig_FROCP}, F-ROCP uses the observed corrupted gradient $\bar{g}_t$ directly. 
We denote the set of in-range time indices as
\begin{equation}
	\mathcal{I} = \{ t \in \{ 1 , \dots, T\} : 0 \le r_t < B \}.
	\label{eq_interval}
\end{equation}

\begin{algorithm}[t]\label{Alg_1}
\caption{Robust Online Conformal Prediction via Filtering (F-ROCP)}
\KwIn{Miscoverage Rate $\alpha$}

Initialize $r_1$\\
\For{$t = 1$ \KwTo $T$}{	

	\textbf{Observe} an input $X_t \in \mathcal{X}$\\
	\textbf{Return} the \textit{prediction set} via \eqref{Eq_Pre_Set}\\
	Receive corrupted feedback $\bar{g}_t = \alpha - \bar{e}_t$\\

	\uIf{$r_t < 0$ or $r_t \ge B$ (Out-of-range threshold)}
	{
		
		Infer true gradient $g_t$ based on \eqref{eq_infer_grad}\\
		Update the threshold $r_{t+1} = r_t - \eta g_t$\\
	}
	\Else({\textit{(in-range threshold)}}){
		Update the threshold $r_{t+1} = r_t - \eta \bar{g}_t$\\
	}
}
\end{algorithm}

\subsection{Miscoverage Analysis}

The miscoverage performance of F-ROCP is characterized by the following theorem.
\begin{theorem}\label{Theo_2}
	Under Assumption~\ref{Assumption_1}, the empirical miscoverage rate under gradient corruption satisfies the upper bound
	\begin{equation}\label{eq_16_New}
		\mathrm{MisCov}(T)
		\le \frac{ B + \eta }{\eta T}
		+ \frac{ \left| \tilde{G}_{0\to1}^{|\mathcal{I}|} - \tilde{G}_{1\to0}^{|\mathcal{I}|} \right| }{T} ,
	\end{equation}
	where $\tilde{G}_{0\to1}^{|\mathcal{I}|} = \left| \{ t \in \mathcal{I} : g_t - \bar{g}_t = 1 \} \right|$ and $\tilde{G}_{1\to0}^{|\mathcal{I}|} = \left| \{ t \in \mathcal{I} : g_t - \bar{g}_t = -1 \} \right|$ denote the number of feedback errors that flip the gradient from $g_t = \alpha$ to $\bar{g}_t = \alpha - 1$ and from $g_t = \alpha - 1$ to $\bar{g}_t = \alpha$, respectively, only among the in-range time indices.
\end{theorem}
\begin{IEEEproof}
	The proof uses an argument similar to that of Theorem \ref{Theo_1} in supplementary material Appendix \ref{App_Theo_1}.
\end{IEEEproof}

The miscoverage bound in Theorem \ref{Theo_2} can be interpreted as follows:
\begin{itemize}
	\item The first term is the same as in the ideal feedback case \eqref{eq_mis_ideal_ocp} and in Theorem \ref{Theo_1}, vanishing as the time horizon $T$ grows.
	\item Meanwhile, similar to the second term in Theorem \ref{Theo_1}, the second term in \eqref{eq_16_New} captures the effect of directional imbalance in the corruption process, but only over the in-range time steps $\mathcal{I}$.
\end{itemize}

Comparing the upper bounds in Theorem \ref{Theo_2} and Theorem \ref{Theo_1}, one cannot conclude that F-ROCP 
necessarily improve miscoverage performance over OCP. However, as will be shown in Sec.~\ref{Sec_6} and Sec.~\ref{Sec_7},  F-ROCP can be shown to  provably improves upon OCP under corrupted feedback, in terms of these bounds, when specializing the general results of the theorems to relevant feedback error models.

\section{Robust Online Conformal Prediction via Active Compensation}\label{Sec_5}
While F-ROCP eliminates the drift in the threshold iterate $r_t$ by filtering corrupted gradients in the out-of-range regime, it still relies on the corrupted gradient $\bar{g}_t$ during in-range rounds.
To address this limitation, we propose \emph{robust OCP via active compensation} (AC-ROCP), which augments F-ROCP with an active compensation step during in-range rounds. 
The protocol is presented next, followed by a formal analysis of its miscoverage performance.

\begin{figure}[t]
	\centering
	\includegraphics[width=0.5\textwidth]{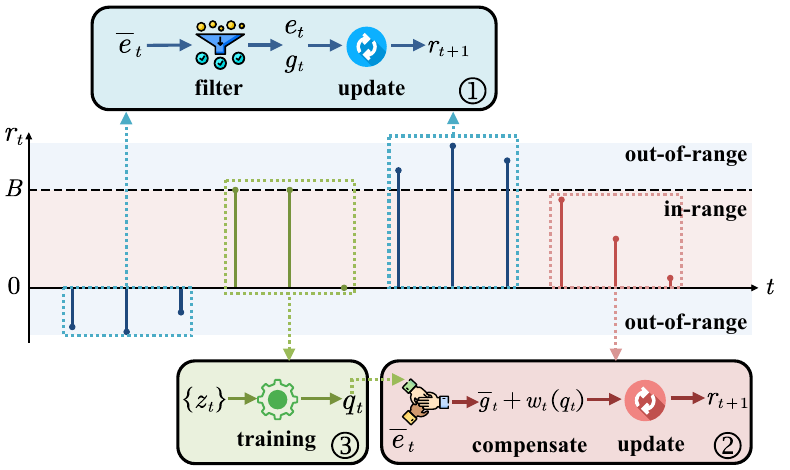}
	\caption{Illustration of AC-ROCP:  
	When the threshold $r_t$ is out-of-range (blue region), the true gradient can be  inferred exactly, and the OCP update can be applied directly. 
	When the threshold is in-range (red region), the possibly corrupted gradient is compensated by a corrective term $w_t(q_t)$ based on the predictor $q_t$, and the compensated gradient is used to update the threshold.}
	\label{Fig_ACROCP}
\end{figure}

\subsection{Protocol}
As shown in Fig. \ref{Fig_ACROCP}, the AC-ROCP protocol consists of three components: 
a \textit{filtering mechanism} (module \ding{172} in Fig. \ref{Fig_ACROCP}) that filters corrupted feedback in the out-of-range regime,
a \textit{compensation mechanism} (module \ding{173} in Fig. \ref{Fig_ACROCP}) that corrects the corrupted gradient during in-range rounds, 
and an \textit{active training} protocol (module \ding{174} in Fig. \ref{Fig_ACROCP}) that acquires signals to train the compensation mechanism.

Since the filtering mechanism is identical to that of F-ROCP, we focus on describing the compensation mechanism and active training protocol in the following.

\subsubsection{Compensation Mechanism}

The compensation mechanism is motivated by \emph{optimistic online convex optimization} \cite{orabona2019modern}, where a learner incorporates a prediction of the next gradient into the update step. 
If the prediction is accurate, this approach can accelerate convergence, leading to lower regret \cite{orabona2019modern}.
Motivated by this principle, AC-ROCP introduces a predictor $q_t$ that estimates the feedback corruption process, which is leveraged to partially compensate the effect of feedback corruption on the threshold update.

The predictor $q_t$ of the feedback error indicator $z_t$  will be discussed for specific corruption models in Sec. \ref{Sec_6} and \ref{Sec_7}. Assuming for now a generic prediction signal, 
AC-ROCP augments F-ROCP by adding a compensation term obtained as a function $w_t(q_t)$ of the prediction $q_t$ to the corrupted gradient $\bar{g}_t$ whenever the threshold is in-range, i.e., if $r_t \in [0, B)$. 
Accordingly, the update rule is
\begin{equation}\label{eq_25}
	r_{t+1} = \left\{ 
		\begin{array}{ll}
			r_t - \eta ( \bar{g}_t + w_t(q_t) ) , & \text{if } 0 \le r_t < B, \\
			r_t - \eta ( \alpha - 0) , & r_t \ge B, \\
			r_t - \eta ( \alpha - 1) , & \text{otherwise},
		\end{array}
	\right.
\end{equation}
Intuitively, the term $w_t (q_t)$ acts as a corrective offset that makes the compensated gradient $\bar{g}_t + w_t (q_t)$ as close as possible to the true gradient $g_t$.

To design the correction, we first consider the ideal case in which the feedback error indicator $z_t$ is known. In this setting, the optimal correction would be
\begin{equation}\label{eq_28}
	\begin{aligned}
		w_t(z_t) & = g_t - \bar{g}_t
		= \bar{e}_t - e_t \\
		&  =  ( 2 \bar{e}_t  - 1 ) z_t ,
	\end{aligned}
\end{equation}
where the second equality  follows from the fact that the error indicator is unchanged when \(z_t=0\), and is flipped when \(z_t=1\), i.e., we have $e_t = (1-z_t)\bar e_t + z_t(1-\bar e_t)$. 
However, since the feedback error indicator $z_t$ is unknown, we propose to replace it with the predictor $q_t$, yielding
\begin{equation}\label{eq_29}
	w_t(q_t)
	= ( 2 \bar{e}_t  - 1 ) q_t .
\end{equation}

\subsubsection{Active Training Protocol}
To train the predictor $q_t$, AC-ROCP requires access to the true feedback error indicator $z_t$. 
However, the true feedback error indicator $z_t$ is unavailable under arbitrary corruption, creating a circular problem: the predictor cannot be trained without clean observations, yet clean observations are precisely what corruption withholds. 

To resolve this issue, we introduce a set of active training rounds $\mathcal{P} \subseteq \{ 1 , \cdots, T \}$, at which AC-ROCP deliberately probes the environment to obtain the true feedback error indicator $z_t$.
Ordering the training rounds as $\mathcal{P} = \{ \tau_1 < \tau_2 < \cdots < \tau_{|\mathcal{P}|} \}$, the threshold at the $i$-th training round, i.e., $t = \tau_{i}$, is set deterministically as
\begin{equation}\label{eq_30}
    r_{\tau_i} = \left\{  
        \begin{array}{ll}
            B, & \text{if } i \leq  \lceil (1-\alpha)|\mathcal{P}| \rceil , \\
            0,          & \text{if } i > \lceil (1-\alpha)|\mathcal{P}| \rceil,
        \end{array}
    \right.
\end{equation}
so that the first $\lceil (1-\alpha)|\mathcal{P}| \rceil$ training steps return a full prediction 
set and the last $\lceil (1-\alpha)|\mathcal{P}| \rceil$ training steps return an empty prediction set.
By construction, the miscoverage rate over the training rounds 
approximately matches the target level $\alpha$, and thus the active training 
rounds do not violate the long-run coverage guarantee.

Furthermore, also by design, during the training rounds the true feedback error indicator $z_t$ can be inferred deterministically as\begin{equation}
	z_{\tau_i} = \left\{ 
		\begin{array}{ll}
		\bar{e}_{\tau_i} , & \text{if } i \leq  \lceil (1-\alpha)|\mathcal{P}| \rceil , \\
		1 - \bar{e}_{\tau_i} , & \text{if } i > \lceil (1-\alpha)|\mathcal{P}| \rceil .
		\end{array}
	\right.
\end{equation}
This recovery is immune to corruption and yields samples of feedback error indicator $z_t$ for training the predictor $q_t$.

In order to keep track of the  time instants $\mathcal{P}$, which  interrupt the standard OGD evolution in
\eqref{eq_25}, we introduce an auxiliary threshold sequence $\{h_t\}_{t\ge1}$
indexed by the original time variable $t$. The sequence is kept unchanged on
training rounds and is updated only on non-training rounds:
\begin{equation}\label{Neq_31}
	h_{t+1} = \left\{ 
		\begin{array}{ll}
			h_t - \eta ( \bar{g}_t + w_t(q_t) ) , & t \notin \mathcal{P}, 0 \le h_t < B, \\
			h_t - \eta ( \alpha - 0) , & t \notin \mathcal{P}, h_t \ge B, \\
			h_t - \eta ( \alpha - 1) , & t \notin \mathcal{P}, h_t < 0 , \\
			h_t , & t \in \mathcal{P} .
		\end{array}
	\right.
\end{equation}
On each non-training round $t \notin \mathcal{P}$, the played threshold is set to
$r_t = h_t$.

Algorithm \ref{Alg_2} summarizes the full AC-ROCP procedure. 
At each round, the learner first determines whether the current step is a training step. 
If so, it plays a boundary threshold via \eqref{eq_30} and uses the resulting signal to train $q_t$. 
Otherwise, it plays the hidden threshold $h_{t}$ and updates the hidden process via \eqref{Neq_31} accordingly.

\begin{algorithm}[t]\label{Alg_2}
\caption{AC-ROCP with Active Training}
\KwIn{Miscoverage Rate $\alpha$, Training Step Set $\mathcal{P}$}

Initialize $r_1$, $h_1$, and $q_1$\\
\For{$t = 1$ \KwTo $T$}{	

	\textbf{Observe} an input $X_t \in \mathcal{X}$\\
	\eIf{$t \in \mathcal{P}$}
	{
		Update the played threshold $r_{t}$ via \eqref{eq_30}
	}
	{
		Update the played threshold $r_{t} = h_{t}$\\
	}
	\textbf{Return} the \textit{prediction set} via \eqref{Eq_Pre_Set}\\
	Receive corrupted feedback $\bar{g}_t = \alpha - \bar{e}_t$\\

	\If{$t \in \mathcal{P}$}
	{
		Train the predictor $q_t$
	}

	Update hidden threshold $h_{t+1}$ via \eqref{Neq_31}\\
	
}
\end{algorithm}

\subsection{Miscoverage Analysis}

The miscoverage performance of AC-ROCP depends on the prediction error, i.e., on the discrepancy between prediction $q_t$ and feedback signal $z_t$, in a way that  is characterized in the following theorem. 
\begin{theorem}\label{Theo_3}
	Under Assumption~\ref{Assumption_1}, with a bounded compensation term $w_t (q_t) \in [-W, W]$, the empirical miscoverage rate of AC-ROCP under gradient corruption satisfies the upper bound
	\begin{subequations}
		\begin{align}
			 \mathrm{MisCov}(T)  &\le \frac{B + \eta(W + 2)}{\eta T} 
			+ \frac{1}{T} \left| \sum_{t \in \mathcal{I}} \left( 2 \bar{e}_t - 1 \right) \left( z_t - q_t \right) \right| \label{eq_28_a} \\
			& \le \frac{B + \eta(W + 2)}{\eta T} 
			+  \frac{1}{T} \sum_{t \in \mathcal{I}} \left|  z_t - q_t \right| ,  \label{eq_28_b}
		\end{align}
	\end{subequations}
	where $\mathcal{I}$ is the set of in-range time indices defined in \eqref{eq_interval}.
\end{theorem}
\begin{IEEEproof}
	Please see supplementary material Appendix \ref{App_Pro:wt_2}.
\end{IEEEproof}

Theorem \ref{Theo_3} shows that the miscoverage of AC-ROCP depends on two terms: 
\begin{itemize}
	\item The first term resembles the base term in the ideal feedback case \eqref{eq_mis_ideal_ocp} and in Theorem \ref{Theo_2}, but with an additional dependence on the compensation bound $W$. 
	This term reflects the variability in the prediction set \eqref{Eq_Pre_Set} caused by the addition of the compensation $w_t(q_t)$.

	\item The second term in \eqref{eq_28_a} captures the correction performance of AC-ROCP, which depends on the prediction error $z_t - q_t$, or on its absolute value $|z_t - q_t|$ in the relaxed bound in \eqref{eq_28_b}.
\end{itemize}

The next two sections study two representative settings for the feedback error, specialize the results in Theorem \ref{Theo_3} to these settings, and provide further insights into the performance of F-ROCP and AC-ROCP.

\section{Independent Stochastic Error Model}\label{Sec_6}
In this section, we adopt an independent stochastic error model.
Then, we specialize AC-ROCP to this setting and derive explicit miscoverage bounds for OCP, F-ROCP and AC-ROCP.

\subsection{Feedback Error Model}

This section assumes that the feedback error indicator $z_t$ follows an i.i.d.\ Bernoulli process
\begin{equation}\label{eq_zt_iid}
    z_t \overset{\mathrm{i.i.d.}}{\sim} \mathrm{Bern}(p), \quad \forall\, t \in \{1, 2, \ldots, T\},
\end{equation}
where $p \in [0, 0.5)$ is the corruption probability. 
This stochastic model is \emph{memoryless} and \emph{stationary}.
It is well-suited for practical settings such as memoryless communication failures over a noisy channel \cite{cover1999elements}.

\subsection{Designing AC-ROCP for the Independent Stochastic Error Model}\label{Sec_6_2}

We now specialize AC-ROCP to the independent stochastic error model. 
The goal is to design the predictor $q_t$ together with the selection of the training rounds $\mathcal{P}$.

\subsubsection{Optimal Predictor under Known \texorpdfstring{$p$}{p}}

To build intuition, suppose first that the feedback corruption probability $p$ is known.
Under this assumption, we aim at minimizing the upper bound in \eqref{eq_28_a} on average, i.e.,
\begin{equation}
    \begin{aligned}
        \min_{q} \! \left\{ \! \mathbb{E}_{z_t} \!\! \left[
          \left| \bigl(2\bar{e}_t \!-\! 1\bigr)(z_t \!-\! q_t) \right|
        \right]  
		\!=\! \left| (1 \!-\! 2 e_t) \right| (p \!-\! (2p \!-\! 1)q_t) \right\}.
    \end{aligned}
\end{equation}
The resulting optimal predictor can be readily derived as 
\begin{equation}\label{eq:q_optimal}
    q_t^* = \frac{p}{2p-1} .
\end{equation}

\subsubsection{Predictor under Unknown \texorpdfstring{$p$}{p}}

In practice, the feedback corruption probability $p$ is unknown. 
As shown in Fig~\ref{fig_iid}, to estimate it, we reserve the first $|\mathcal{P}|$ rounds as a training phase, i.e.,
\begin{equation}\label{eq:training_rounds}
    \mathcal{P} = \{1, 2, \ldots, |\mathcal{P}|\} .
\end{equation}
During this training phase, the true feedback error indicators  $\{z_t\}_{t \in \mathcal{P}}$ are collected and used to estimate the feedback corruption probability $p$.

Specifically, we employ a truncated Krichevsky--Trofimov (KT) estimator \cite{KT1981TIT} to track the corruption probability $p$ online during the training phase as
\begin{equation}\label{eq_training_KT}
    \hat{p}_t = \min \left\{ \frac{0.5 + \sum_{i \in \mathcal{P}_t} z_i}{|\mathcal{P}_t| + 1} , \gamma \right\} ,
\end{equation}
where $\mathcal{P}_t = \{i \in \mathcal{P} : i < t\}$ is the set of training rounds before time $t$.
and $\gamma \in (0, 0.5)$ is a hyperparameter.

\begin{figure}
	\centering
	\includegraphics[width=0.475\textwidth]{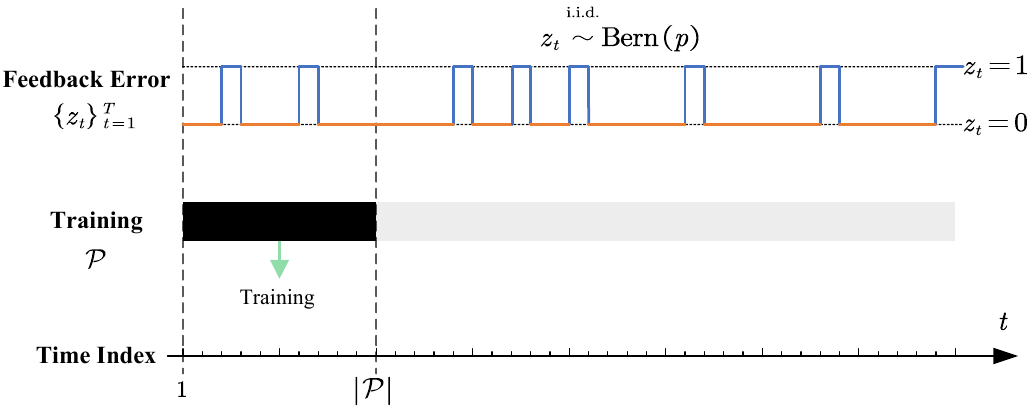}
	\caption{Illustration of the independent stochastic feedback error model shown in Sec. \ref{Sec_6}. The time horizon is partitioned into two intervals, where the first $|\mathcal{P}|$ rounds are used for training.}
	\label{fig_iid}
\end{figure}

\subsection{Performance Analysis}

We now derive explicit miscoverage bounds for all three methods under the independent stochastic error model.

\begin{corollary}[OCP under Independent Stochastic Error]\label{Coro_3_1}
	Specializing Theorem~\ref{Theo_1} to the independent stochastic error model, the empirical miscoverage rate of OCP satisfies, with probability at least $1 - \delta$, the upper bound
    \begin{equation}
        \mathrm{MisCov}(T)
        \le \frac{B + \eta}{\eta T}
        + 2
          \left(p + \sqrt{\frac{\log(1/\delta)}{2T}}\right).
    \end{equation}
\end{corollary}
\begin{IEEEproof}
    Please see supplementary material Appendix~\ref{App_Coro_3_1}.
\end{IEEEproof}

Corollary~\ref{Coro_3_1} establishes that the miscoverage error of OCP scales linearly with the feedback corruption probability $p$.

\begin{corollary}[F-ROCP under Independent Stochastic Error]\label{Coro_4_1}
	Specializing Theorem~\ref{Theo_2} to the independent stochastic error model, the empirical miscoverage rate of F-ROCP satisfies, with probability at least $1 - \delta$, the upper bound
	    \begin{equation}
        \mathrm{MisCov}(T)
        \le \frac{B + \eta}{\eta T}
        + \frac{|\mathcal{I}|}{T}\left(p + \sqrt{\frac{\log(1/\delta)}{2|\mathcal{I}|}}\right).
    \end{equation}
\end{corollary}
\begin{IEEEproof}
	It can be derived similarly to Corollary \ref{Coro_3_1} in supplementary material Appendix \ref{App_Coro_3_1}.
\end{IEEEproof}

Compared with Corollary~\ref{Coro_3_1}, F-ROCP replaces the term $\mathcal{O}(2p)$ by $\mathcal{O}(p|\mathcal{I}|/T)$. 
Since the inequality $|\mathcal{I}| \le T$ holds by construction, the bound is no worse than that of OCP. 
Nevertheless, the residual factor $|\mathcal{I}| p /T$ shows that corruption concentrated within the in-range regime can still affect coverage.

\begin{corollary}[AC-ROCP under Independent Stochastic Error]\label{Coro_5_1}
    Specializing Theorem~\ref{Theo_3} to the independent stochastic error model with compensation design in Sec. \ref{Sec_6_2}, the empirical miscoverage rate of AC-ROCP satisfies, with probability at least $1 - \delta$, the upper bound
    \begin{equation}
        \mathrm{MisCov}(T)
        \le \frac{B + \eta (W+2)}{\eta T} + H_1 + H_2,
    \end{equation}
    where $W = |p/(2p-1)|$,
    \begin{align}
        H_1 &= \frac{1}{(1-2p)\,T}
               \sqrt{\frac{|\mathcal{I}|\log(4/\delta)}{2}}, \text{ and } \nonumber \\
        H_2 &= \frac{|\mathcal{I}|}{(1 - 2\max\{p,\hat{p}\})^2\,T}
               \left(\sqrt{\frac{\log(4/\delta)}{2(|\mathcal{P}|+1)}}
               + \frac{1}{2(|\mathcal{P}|+1)}\right). \nonumber
    \end{align}
\end{corollary}
\begin{IEEEproof}
    Please see supplementary material Appendix~\ref{App_Coro_5_1}.
\end{IEEEproof}

The bound in Corollary~\ref{Coro_5_1} comprises two distinct error terms $H_1$ and $H_2$:
\begin{itemize}
	\item Term $H_1$ reflects the irreducible compensation mismatch that persists even when $p$ is known exactly: it decays as $\mathcal{O}(\sqrt{|\mathcal{I}|}/T)$ and grows as $p$ tends to $0.5$, where corruption becomes maximally ambiguous. 
	\item Term $H_2$ quantifies the additional degradation due to estimating $p$ from the training phase via the KT estimator. 
	This term decays as $\mathcal{O}(|\mathcal{I}|/(T\sqrt{|\mathcal{P}|}))$, and can be made negligible by allocating a sufficiently large training phase $|\mathcal{P}|$.
\end{itemize}

Compared with OCP and F-ROCP, AC-ROCP achieves strictly superior miscoverage performance: as $T \to \infty$, the error terms $H_1$ and $H_2$ both vanish, provided the training phase grows accordingly, yielding asymptotically vanishing miscoverage error even under corrupted feedback.

\section{Arbitrary Feedback Errors with Memory Bounds}\label{Sec_7}
While the previous section focused on a stochastic memoryless feedback error model, in this section we consider a complementary setting characterized by an arbitrary, potentially adversarial, corruption model with memory bounds.
We first describe the model and then specialize AC-ROCP to this setting, deriving explicit miscoverage bounds for OCP, F-ROCP and AC-ROCP.

\begin{figure}
	\centering
	\includegraphics[width=0.475\textwidth]{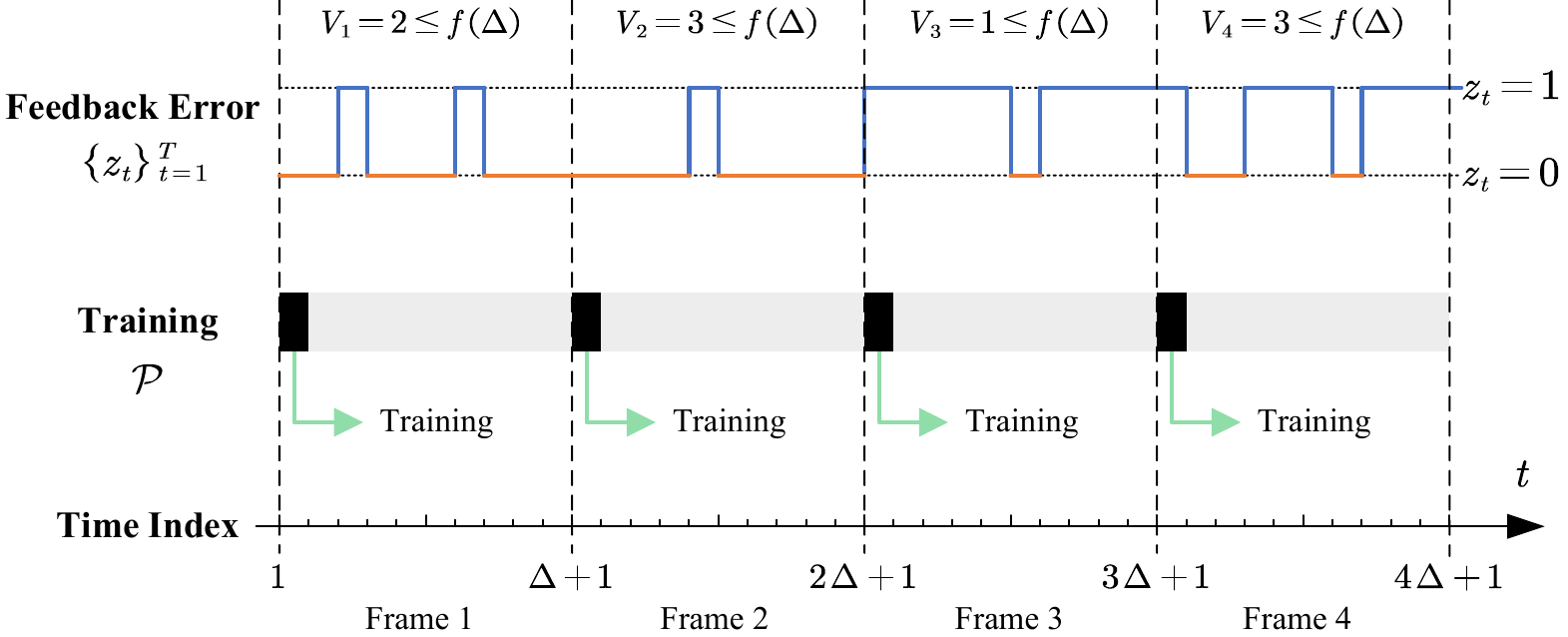}
	\caption{Illustration of the deterministic feedback error model with memory bounds shown in Sec. \ref{Sec_7}. The time horizon is partitioned into intervals of length $\Delta$, within which the feedback error indicator $z_t$ can switch at most $f(\Delta)$ times.}
	\label{fig_arbitary}
\end{figure}

\subsection{Feedback Error Model}
We consider an arbitrary, potentially adversarial, setting in which the frequency of corruption-state changes is bounded over time.
Specifically, as shown in Fig.~\ref{fig_arbitary}, the time horizon $\{1, 2, \ldots, T\}$ is partitioned into $\lceil T/\Delta \rceil$ consecutive, non-overlapping intervals of length $\Delta$. 
The parameter $\Delta$ governs the granularity of the memory structure.
Within each interval $\{(i-1)\Delta + 1, \ldots, i\Delta\}$, referred to as a \textit{frame}, the total variation of the feedback error indicator is constrained by a non-decreasing function $f(\cdot)$ with $f(0) = 0$.
Mathematically, we assume that the inequality
\begin{equation}\label{Neq_13}
    V_i = \sum_{t=(i-1)\Delta+1}^{i\Delta}
    \mathbbm{1}\{z_t \neq z_{(i-1)\Delta+1}\} \le f(\Delta) ,
\end{equation}
holds for every frame.

The constraint limits how many times the corruption state may change within any single interval. 
This error model thus bounds the memory of the corruption process by imposing that, within each block of length $\Delta$, the number of state transitions is at most $f(\Delta)$.

This model captures practically relevant disturbances such as sticky sensor faults, prolonged fading blocks, and energy-constrained adversaries whose attack budget per window is limited.

\subsection{Designing AC-ROCP for the Arbitrary Feedback Model with Memory Bounds}
In this setting, the feedback error indicator sequence $z_t$ may be chosen adversarially, subject only to the variation constraint \eqref{Neq_13} within each interval. 
The local regularity imposed by this constraint suggests a natural predictor: observe the corruption state at the beginning of each interval and hold it fixed throughout.
As we will discuss in the next subsection, this approach generalizes to other settings in which feedback errors have memory and thus can potentially be predicted.

Formally, as shown in Fig.~\ref{fig_arbitary}, we designate the first round of each interval as a probing step, yielding the set of training time steps
\begin{equation}\label{eq:probing_interval}
    \mathcal{P} = \{1,\, \Delta+1,\, 2\Delta+1,\, \ldots\}.
\end{equation}
For each interval $m = 1, 2, \ldots$, the predictor is then set as $q_t = z_{(m-1)\Delta+1}$.
Since the variation budget $f(\Delta)$ limits the number of corruption-state switches within any interval, the cumulative prediction error over each block is bounded directly by $f(\Delta)$, making this \textit{sense-and-hold} strategy both simple and analytically tractable.

\subsection{Performance Analysis}

We now characterize the miscoverage bounds for all three methods under this model. 
For compactness, define
\begin{equation}\label{Eq_Hat_F_Delta}
    \hat{f}(\Delta) = \max\{f(\Delta),\, \Delta - f(\Delta)\},
\end{equation}
which captures the worst-case imbalance between the corrupted and ideal rounds within any interval of length $\Delta$.

\begin{corollary}[OCP under Arbitrary Feedback Errors with Memory Bounds]
\label{Coro_3_2}
    Specializing Theorem~\ref{Theo_1} to this error model, the empirical miscoverage rate of OCP satisfies the upper bound
    \begin{equation}\label{Eq_Coro_3_2}
        \mathrm{MisCov}(T) \le \frac{B + \eta}{\eta T}
        + \frac{2 \hat{f}(\Delta)}{\Delta} + \frac{2 \hat{f}(\Delta)}{T} .
    \end{equation}
\end{corollary}
\begin{IEEEproof}
    Please see supplementary material Appendix~\ref{App_Coro_3_2}.
\end{IEEEproof}

Corollary~\ref{Coro_3_2} shows that the miscoverage of OCP is governed by $\hat{f}(\Delta)/\Delta$, i.e., the normalized worst-case variation within a window. 
Critically, since the corruption-dependent term does not vanish as $T \to \infty$, the miscoverage error may persist indefinitely under adversarial corruption regardless of the sample size.

\begin{corollary}[F-ROCP under Arbitrary Feedback Errors with Memory Bounds]\label{Coro_4_2}
	Specializing Theorem~\ref{Theo_2} to this error model, the empirical miscoverage rate of F-ROCP satisfies
	\begin{equation}\label{eq_new_41}
		\mathrm{MisCov}(T) \le   \frac{ B + \eta}{\eta T} 
		+  \frac{\hat{f}(\Delta)}{\Delta} + \frac{\hat{f}(\Delta)}{T} ,
	\end{equation}
	where $\hat{f}(\Delta)$ is defined in \eqref{Eq_Hat_F_Delta}.
\end{corollary}
\begin{IEEEproof}
	It can be similarly derived as Corollary \ref{Coro_3_2} in supplementary material Appendix \ref{App_Coro_3_2}.
\end{IEEEproof}

Compared with Corollary~\ref{Coro_3_2}, the bound \eqref{eq_new_41} confirms that filtering attenuates the impact of bounded-memory corruption. 
Nevertheless, the leading term $\hat{f}(\Delta)/\Delta$ remains non-vanishing, indicating that F-ROCP cannot fully eliminate the miscoverage penalty under adversarial corruption.

\begin{corollary}[AC-ROCP under Arbitrary Feedback Errors with Memory Bounds]\label{Coro_5_2}
	Specializing Theorem~\ref{Theo_3} to this error model, the empirical miscoverage rate of AC-ROCP satisfies the upper bound
	\begin{equation}
		\begin{aligned}
			\mathrm{MisCov}(T) &  \le \frac{B + 3 \eta}{\eta T}
			+ \frac{f (\Delta)}{\Delta} + \frac{f(\Delta)}{T} .
		\end{aligned}
	\end{equation}
\end{corollary}
\begin{IEEEproof}
	Please see supplementary material Appendix \ref{App_Coro_5_2}.
\end{IEEEproof}

Corollary~\ref{Coro_5_2} reveals a qualitative improvement over the preceding two results. When the error variation bound $f(\Delta)$ grows sublinearly in $\Delta$, the dominant term $f(\Delta)/\Delta$ vanishes asymptotically, and thus AC-ROCP achieves miscoverage control that improves with $\Delta$ rather than saturating. 
Moreover, unlike OCP and F-ROCP, the bound depends on the term $f(\Delta)$ directly, rather than on the more conservative quantity $\hat{f}(\Delta)$, reflecting the gain from active compensation.

\section{Experiments}\label{Sec_8}
In this section, we evaluate the proposed methods under two corruption models: the independent stochastic error model studied in Sec.~\ref{Sec_6} and a Markov error model with memory, representing settings with memory, as investigated in Sec.~\ref{Sec_7}.

\subsection{Settings}

\subsubsection{Datasets}
We consider both classification and regression tasks, using two representative datasets.

\textbf{Classification task on CIFAR-100:}
We consider a multi-class classification task on the CIFAR-100 dataset \cite{krizhevsky2009learning}, which contains $60{,}000$ images from $100$ classes.
We initialize the base predictor with a pre-trained ResNet-18 model \cite{he2016deep} and further train it on the CIFAR-100 training set.
The CIFAR-100 test set, which contains 10,000 images, is reserved for online calibration under corrupted feedback, yielding a time horizon of $T=10{,}000$.

\textbf{Regression task on AVA:}
We also consider a regression task on the Aesthetic Visual Analysis (AVA) dataset \cite{murray2012ava}, which consists of images annotated with aesthetic ratings from 1 to 10.
Each rating is obtained by averaging the scores provided by annotators, and the goal is to predict the average aesthetic score of a given image.
We implement the regression model using a VGG16 network pre-trained on ImageNet, whose final fully connected layer is replaced by a single-output layer producing a continuous predicted score \cite{einbinder2024label}.
We use $34{,}000$ images to adapt the pre-trained model and reserve the remaining $15{,}557$ images for online calibration under corrupted feedback, corresponding to a horizon of $T = 15{,}557$ \cite{einbinder2024label}.

\subsubsection{Corruption Error Models}
We consider the following two corruption models.

\textbf{Independent stochastic error model:}
As described in Sec.~\ref{Sec_6}, in this model, the feedback error indicator $z_t$ follows an i.i.d. Bernoulli process with corruption probability $p = \Pr \{ z_t = 1\}$.

\textbf{Markov error model with memory:}
In this model, the feedback error indicator $z_t$ evolves according to a two-state Markov chain \cite{levin2017markov}, with transition probabilities
\begin{equation}
	\begin{aligned}
		p_{01} = \Pr(z_{t+1}=1 \mid z_t=0), \text{ and }
		p_{10} = \Pr(z_{t+1}=0 \mid z_t=1).
	\end{aligned}
\end{equation}
The corresponding stationary distribution is
\begin{equation}
	\begin{aligned}
		\pi_0 = \frac{p_{10}}{p_{01}+p_{10}}, \text{ and }
		\pi_1 = \frac{p_{01}}{p_{01}+p_{10}},
	\end{aligned}
\end{equation}
where $\pi_1$ denotes the stationary corruption probability.
In all simulations, the Markov chain is initialized from its stationary distribution.
Following standard Markov chain theory \cite{levin2017markov}, the memory length (also referred to as the relaxation time) is defined as
\begin{equation}
	\mathcal{M} = \frac{1}{p_{01} + p_{10}} ,
\end{equation}
which quantifies the temporal dependence of the corruption process.
When $\mathcal{M}$ is small, the corruption process evolves rapidly and exhibits weak memory, whereas larger values of $\mathcal{M}$ correspond to slower evolution and stronger memory.

\subsection{Independent Stochastic Error}

In this subsection, we evaluate the proposed methods under the independent stochastic error model on the CIFAR-100 classification task.

We adopt the score function $S(X_t, y) = 1 - [f(X_t)]_y$, where $[f(X_t)]_y$ denotes the softmax probability assigned by the pre-trained model to class $y$ for input $X_t$ at time step $t$ \cite{angelopoulos2024online}.
The target miscoverage level is set to $\alpha = 0.1$.
Following Sec.~\ref{Sec_6}, training steps are scheduled according to \eqref{eq:training_rounds} with $|\mathcal{P}| = 50$.

For AC-ROCP, we consider two implementations.
The first is the KT-based scheme labeled as ``AC-ROCP, KT'', which estimates the corruption probability $\hat{p}$ from the training data, introduced in Sec.~\ref{Sec_6}.
The second is an oracle variant (denoted by ``AC-ROCP, Oracle''), which has direct access to the true corruption probability $p$.
Comparing these two variants allows us to quantify the gap between the practical estimator-based scheme and its idealized counterpart.
The OCP with ideal feedback, referred to as ``OCP, Ideal'', is also included as a benchmark, serving as an upper bound on performance.

\subsubsection{Coverage and Set Size over Time}
In Fig.~\ref{Fig_0b}, we report the average coverage and average set size of the considered methods as functions of the time horizon $T$, where the shaded regions denote the standard deviation across $10^4$ independent trials.

In Fig.~\ref{Fig_0b} left, conventional OCP under corrupted feedback becomes increasingly conservative over time, with the coverage rising well above the target level $1-\alpha$ and stabilizing close to one.
F-ROCP alleviates this effect and yields substantially improved coverage, although it still exhibits noticeable over-coverage.
By contrast, both variants of AC-ROCP track the target coverage closely; in particular, the oracle variant nearly overlaps with the ideal feedback benchmark throughout the horizon.
The KT-based variant exhibits an initial full-coverage regime with negligible deviation across trials, because of the deterministic active training phase.
Once training completes, the KT variant converges to the target level and aligns with the oracle variant, indicating that the estimate $\hat{p}$ rapidly becomes accurate enough to match the oracle's calibration behavior.
Across all methods, the shaded regions narrow as $T$ grows, reflecting the expected concentration of the empirical coverage around its mean as more samples accumulate.

Fig.~\ref{Fig_0b} right shows the corresponding average set size.
The corrupted feedback OCP baseline produces excessively large prediction sets, which approach the full label space.
Since CIFAR-100 contains only $100$ classes, this behavior indicates extreme conservativeness.
F-ROCP significantly reduces the set size, but it remains more conservative than the ideal feedback benchmark.
The proposed AC-ROCP achieves a more favorable trade-off, maintaining coverage close to the target while producing substantially smaller sets than OCP and F-ROCP.
We also observe that the KT-based variant yields slightly larger sets than the oracle version during the initial stage, again due to the early active training mechanism.

Overall, these results highlight the advantage of active compensation under stochastic i.i.d. corrupted feedback.
Unlike OCP and F-ROCP, which remain conservative throughout the horizon, AC-ROCP is able to maintain near-target coverage while keeping the prediction sets comparatively informative.

\begin{figure}[t]
	\centering
	\subfigure[]{\label{Fig_Exp1_2_a}
		\includegraphics[width=0.4\textwidth]{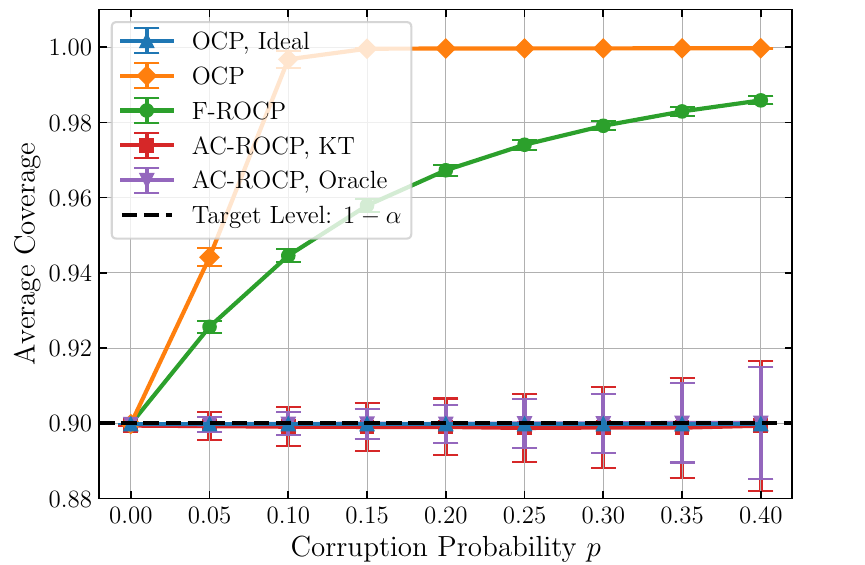}
	}
	\subfigure[]{\label{Fig_Exp1_2_b}
		\includegraphics[width=0.4\textwidth]{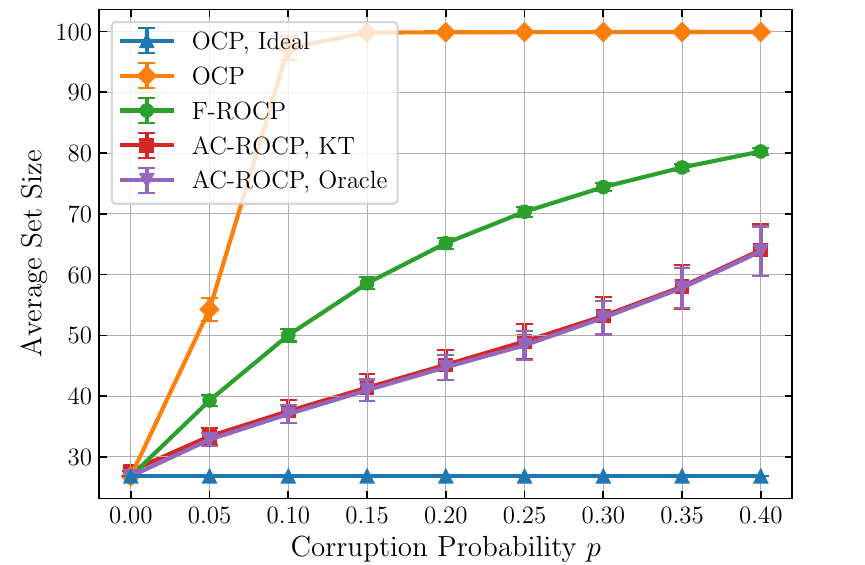}
	}
	\caption{Effect of the corruption probability $p$ on the performance of the different methods with $T = 10{,}000$ and $\alpha = 0.1$. 
	Results are averaged over $10^4$ independent trials, with the shaded region denoting a two-sided interval with half-width equal to one standard deviation.
	(a) Coverage versus $p$. (b) Average set size versus $p$.}
	\label{Fig_Exp1_2}
\end{figure}

\subsubsection{Impact of the Corruption Level}
In Fig.~\ref{Fig_Exp1_2}, we study the effect of the corruption probability $p$ on the average coverage and average set size of the considered methods, where the error bars denote the standard deviation across $10^4$ trials.

As shown in Fig.~\ref{Fig_Exp1_2_a}, the average coverage of OCP rises far above the target level and eventually approaches one as $p$ increases. 
This behavior is expected: a larger corruption probability corresponds to more frequent feedback flips, and as $p$ approaches $0.5$, the feedback becomes maximally corrupted and therefore substantially less informative for calibration. 
F-ROCP mitigates this degradation and yields lower coverage than OCP across all values of $p$, but still exhibits a clear upward drift as the corruption level increases. 
By contrast, both AC-ROCP variants remain tightly concentrated around the target level $1-\alpha$ across the entire range of $p$, indicating that the active compensation mechanism is both effective and reliable even under severe corruption. 
Nevertheless, as $p$ increases, the standard deviation of both AC-ROCP variants grows moderately, which is expected since a higher corruption probability renders the feedback less informative and hence leads to increased variability in the threshold updates.

The same trend is reflected in Fig.~\ref{Fig_Exp1_2_b}, which reports the average set size. 
As $p$ increases, all methods operating under corrupted feedback require larger prediction sets in order to maintain coverage. 
This increase is most pronounced for OCP, whose set size rapidly approaches the entire label space. 
F-ROCP improves upon OCP but still incurs substantial growth in set size as the corruption level rises. 
In comparison, both AC-ROCP variants achieve significantly smaller prediction sets across all values of $p$, with AC-ROCP KT closely matching the oracle variant, suggesting that the KT predictor reliably approximates the oracle corruption estimate. 
The growing error bars of both AC-ROCP variants at higher $p$ are consistent with the coverage observations in Fig.~\ref{Fig_Exp1_2_a}, reflecting the increased variability caused by less informative feedback.

Taken together, Fig.~\ref{Fig_Exp1_2} shows that the proposed method is considerably more robust to increasing corruption than the two corrupted-feedback baselines.

\subsection{Markov Error with Memory}

In this subsection, we evaluate the proposed methods under the Markov feedback error model with memory bounds on the AVA regression task.
We adopt the absolute error as the non-conformity score, namely, $S(X_t, y) = |f_t(X_t) - y|$, where $f_t(X_t)$ denotes the predicted score at time $t$.
The target miscoverage level is set to $\alpha = 0.1$.
Training steps are scheduled according to \eqref{eq:probing_interval}, and the interval between two consecutive training steps is referred to as the \emph{inter-training interval}.
In all simulations, we set $p_{01} = p_{10}$, which yields a stationary corruption rate $\pi_1 = 0.5$ and memory length $\mathcal{M} = 1/(2p_{01})$.

\begin{figure}[t]
	\centering
	\subfigure[]{\label{Fig_Exp2_1_a}
		\includegraphics[width=0.4\textwidth]{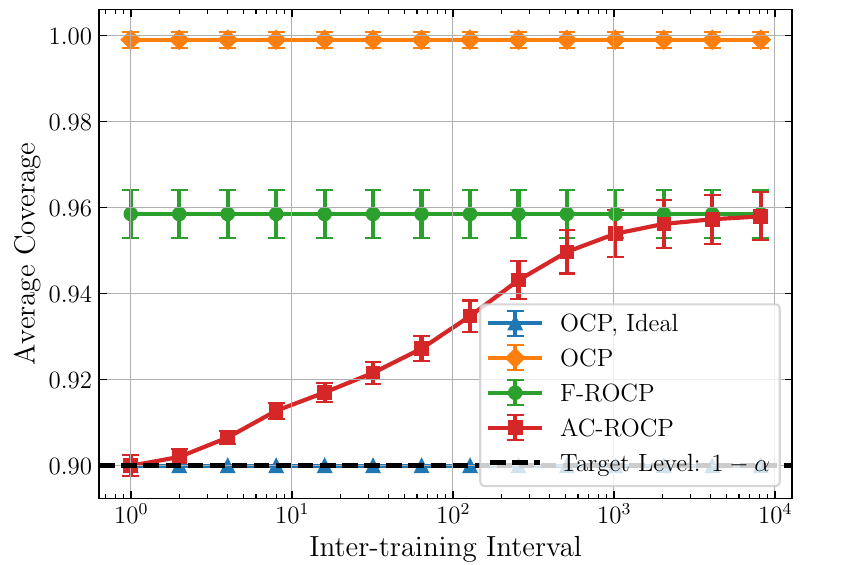}
	}
	\subfigure[]{\label{Fig_Exp2_1_b}
		\includegraphics[width=0.4\textwidth]{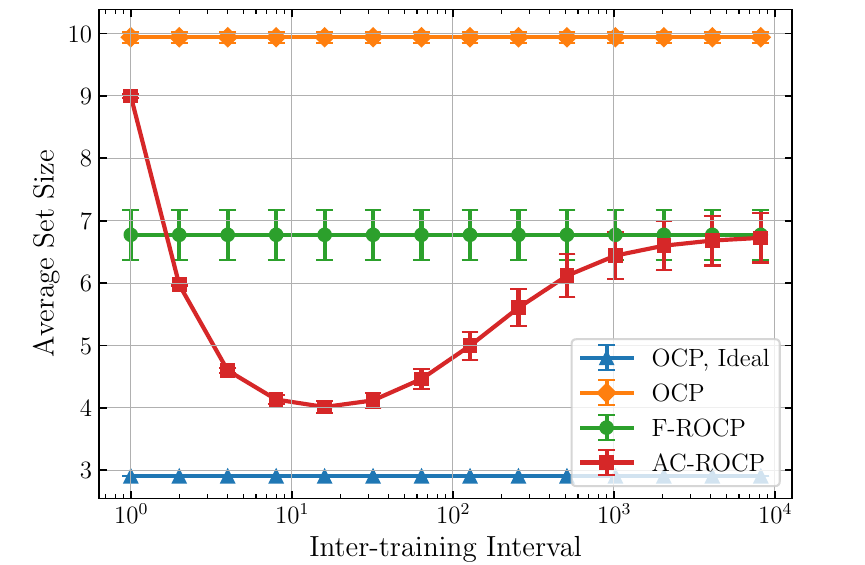}
	}
	\caption{Effect of the inter-training interval on the performance of the different methods with memory length $\mathcal{M} = 100$ and $\alpha = 0.1$. 
	Results are averaged over $10^3$ independent trials, with the shaded region denoting a two-sided interval with half-width equal to one standard deviation.
	(a) Long-run coverage versus inter-training interval. 
	(b) Average set size versus inter-training interval.}
	\label{Fig_Exp2_1}
\end{figure}

\subsubsection{Impact of the Inter-Training Interval}
In Fig.~\ref{Fig_Exp2_1}, we evaluate the effect of the inter-training interval on the long-run coverage and average set size of the considered methods for memory length $\mathcal{M}=100$, where the error bars denote the standard deviation across $10^3$ trials.
As shown in Fig.~\ref{Fig_Exp2_1_a}, the performance of OCP and F-ROCP is essentially unchanged across inter-training intervals, since neither method uses the inserted training steps to learn a predictor.
In contrast, AC-ROCP is sensitive to this design parameter.
When the inter-training interval is small, AC-ROCP achieves long-run coverage close to the target level.
As the inter-training interval increases, the coverage of AC-ROCP gradually drifts above the target, indicating increasingly conservative behavior, while its error bars also become larger.
This is because when the inter-training interval becomes too large relative to the memory length $\mathcal{M}$, the learned predictor becomes increasingly stale and less accurate in tracking the current corruption state.
Consequently, both the average compensation bias and the variability of the compensation error across independent trials increase.

Fig.~\ref{Fig_Exp2_1_b} reports the corresponding average set size.
The set sizes of OCP and F-ROCP remain nearly constant over all inter-training intervals.
In contrast, AC-ROCP exhibits a clear U-shaped trend. 
When the inter-training interval is small, the training stage is invoked frequently and tends to produce larger prediction sets. 
As the interval increases, the training stage becomes less frequent, and the active compensation based updates can still effectively calibrate the predictions, leading to a smaller average set size.
However, once the inter-training interval becomes too large, the information used by the active compensation mechanism becomes stale.
Consequently, the updates become less effective in correcting the corruption, and the average set size increases again.
We also observe that the error bars of AC-ROCP become larger as the inter-training interval increases, which is consistent with the coverage observations in Fig.~\ref{Fig_Exp2_1_a}.

Overall, Fig.~\ref{Fig_Exp2_1} illustrates the trade-off induced by the inter-training interval in AC-ROCP.
Smaller intervals improve tracking of the corruption process but incur greater training overhead, whereas larger intervals reduce training overhead at the cost of degraded calibration and larger prediction sets.

\begin{figure}[t]
	\centering
	\subfigure[]{\label{Fig_Exp2_2_a}
		\includegraphics[width=0.4\textwidth]{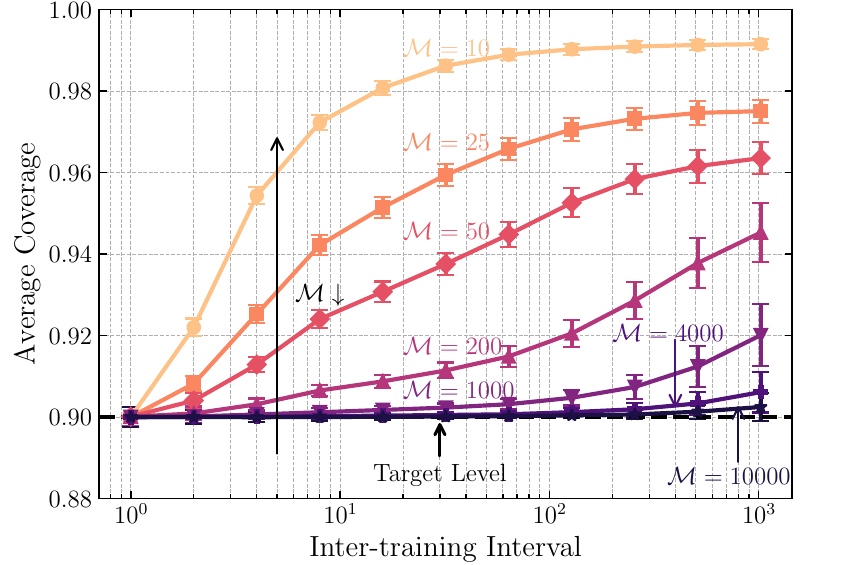}
	}
	\subfigure[]{\label{Fig_Exp2_2_b}
		\includegraphics[width=0.4\textwidth]{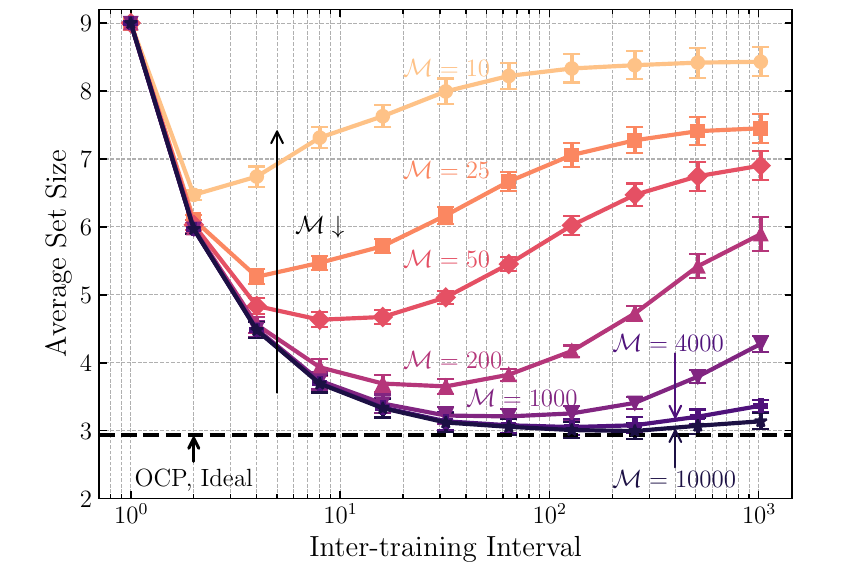}
	}
	\caption{Effect of the memory length $\mathcal{M}$ on the performance of the different methods with $\alpha = 0.1$ and $T = 15,557$. 
	Results are averaged over $10^3$ independent trials, with the shaded region denoting a two-sided interval with half-width equal to one standard deviation.
	(a) Coverage versus inter-training interval. 
	(b) Average set size versus inter-training interval.}
	\label{Fig_Exp2_2}
\end{figure}

\subsubsection{Impact of the Memory Length}
In Fig.~\ref{Fig_Exp2_2}, we study the effect of the memory length $\mathcal{M}$ on the average coverage and average set size of the considered methods, where the error bars denote the standard deviation across $10^3$ trials.

The main observation is that longer memory makes the corruption process easier to track.
As shown in Fig.~\ref{Fig_Exp2_2_a}, when the inter-training interval is very small, the average coverage remains close to the target level for all values of $\mathcal{M}$.
As the inter-training interval increases, the coverage gradually drifts above the target, indicating increasingly conservative behavior.
This drift is more pronounced for shorter memory lengths, whereas larger values of $\mathcal{M}$ allow the coverage to remain closer to the target over a wider range of inter-training intervals.
This behavior is expected, since a longer-memory corruption process evolves more slowly and can therefore be tracked more reliably even when training is performed less frequently.

Fig.~\ref{Fig_Exp2_2_b} reports the corresponding average set size.
For all memory lengths, a similar U-shaped trend is observed. 
This suggests the existence of an approximately optimal inter-training interval that provides a favorable trade-off between calibration accuracy and prediction efficiency.
Moreover, as the memory length increases, the minimizer of the average set size shifts to the right. This indicates that a longer-memory corruption process allows for a larger inter-training interval.
In addition, as the memory length increases, the minimum average set size approaches that of OCP under ideal feedback.

Taken together, Figs.~\ref{Fig_Exp2_2_a} and \ref{Fig_Exp2_2_b} demonstrate that longer memory allows less frequent training while still maintaining coverage close to the target level and relatively small, hence more informative, prediction sets.

\section{Conclusion}\label{Sec_9}

We have studied online conformal prediction (OCP) \cite{gibbs2021adaptive,gibbs2024conformal,bhatnagar2023improved,zecchin2024localized,wu2025error,xu2023conformal} under corrupted feedback and shown that the calibration guarantees of standard OCP can break down when the feedback used for threshold updates is corrupted. 
To mitigate this issue, we have proposed two robust schemes, F-ROCP and AC-ROCP, and established explicit empirical miscoverage guarantees.
We have further specialized these guarantees under stochastic corruption model and bounded-memory corruption model.
Experiments on real-world datasets demonstrated that the proposed methods achieve substantially improved calibration and smaller prediction sets than baseline OCP under corrupted feedback. 
Future work may consider robust OCP under corrupted non-conformity score feedback, more general adaptive risk control with corrupted feedback, and multi-expert enhanced OCP under corrupted feedback.

\footnotesize
\bibliographystyle{IEEEtran}
\bibliography{IEEEabrv,./ref.bib}

\newpage
\normalsize
\appendices
\input{Appendix_OCP.tex}

\end{document}

%% file: Appendix_OCP.tex
\setcounter{page}{1}
\begin{figure*}
	\begin{center}
		{\Huge Online Conformal Prediction with Corrupted \\\vspace{0.3em}Feedback}
		
		\vspace{1em}
		{\large Bowen Wang, \textit{Graduate Student Member, IEEE}, Matteo Zecchin, \textit{Member, IEEE},\\ 
			\vspace{0.2em}
			and Osvaldo Simeone, \textit{Fellow, IEEE}.}
		
		\vspace{1em}
		{\large \textit{(Supplementary Material)}}
	\end{center}
\end{figure*}

\section{Proof of Theorem \ref{Theo_1}}\label{App_Theo_1}
\setcounter{equation}{0}
\renewcommand*{\theequation}{A-\arabic{equation}}

We first prove the lemma below that shows the boundedness of the threshold $r_t$.
Lemma \ref{Lem_3_1} is essential in the proof of Theorem \ref{Theo_1}.
\begin{lemma}\label{Lem_3_1}
	Under Assumption~\ref{Assumption_1}, the threshold $r_t$ satisfies, for all $t$,
	\begin{equation}
		-\eta\alpha \left(1 + G_{1 \to 0}^T\right) \;\le\; r_t \;\le\; B + \eta(1 - \alpha)\left(1 + G_{0 \to 1}^T\right).
	\end{equation}
\end{lemma}
\begin{IEEEproof}
	We establish the upper bound; the lower bound follows by a symmetric argument. If $r_t \le B$, the bound holds trivially, so we consider the case $r_t > B$.
	
	Let $\tau < t$ be the most recent time index at which $r_\tau \le B$ and $r_{\tau+1} > B$; such a $\tau$ exists since $r_1 = 0 \le B$. By the update rule \eqref{OCP_Noisy}, the largest single-step increase from $r_\tau \in [0, B]$ is realized when $\bar{g}_\tau = \alpha - 1$, yielding
	\begin{equation}
		r_{\tau+1} \;=\; r_\tau - \eta(\alpha - 1) \;\le\; B + \eta(1 - \alpha).
	\end{equation}
	
	For every subsequent step $j \in (\tau, t]$, the definition of $\tau$ ensures $r_j > B \ge s_j$, so the true feedback error indicator satisfies $e_j = 0$ and the true gradient is $g_j = \alpha$. 
	A ideal update thus shrinks $r_j$ by $\eta\alpha$; only a corrupted flip $\bar{g}_j = \alpha - 1$ can increase it, and each such flip contributes $\eta (1 - \alpha)$ increment, so $j \in (\tau, t]$ contains at most $G_{0 \to 1}^T$ flips. 
	Summing the increments gives
	\begin{equation}
		\begin{aligned}
			r_t & \le r_{\tau+1} + \eta(1 - \alpha)\, G_{0 \to 1}^T \\
		& \le B + \eta(1 - \alpha)(1 + G_{0 \to 1}^T) .
		\end{aligned}
	\end{equation}
	
	The lower bound follows by the symmetric argument.
\end{IEEEproof}

We now prove Theorem \ref{Theo_1}.
Unrolling the recursion in~\eqref{OCP_Noisy}, the threshold at time $T+1$ can be decomposed as
\begin{equation}\label{Eq_App_2_1}
	\begin{aligned}
		r_{T+1} & = r_{1} - \eta \sum_{t=1}^{T} \bar{g}_t \\
		& = r_{1} - \eta \sum_{t=1}^{T} {g}_t + \eta \sum_{t \in \bar{\mathcal{G}}} \left( g_t - \bar{g}_t \right) ,
	\end{aligned}
\end{equation}
where $\bar{\mathcal{G}}$ is defined in \eqref{Eq_11}, denoting the set of time indices where the feedback is flipped.

Rearranging \eqref{Eq_App_2_1} and dividing by $\eta T$ yields
\begin{equation}
	\begin{aligned}
		\frac{1}{T} \sum_{t=1}^{T} {g}_t & = \frac{r_1 - r_{T+1}}{\eta T} + \frac{\sum_{t \in \bar{\mathcal{G}}} \left( g_t - \bar{g}_t \right)}{T} \\
		& \mathop = \limits^{\text{(a)}} \frac{r_1 - r_{T+1}}{\eta T} + \frac{G_{0 \to 1}^T - G_{1 \to 0}^T}{T} ,
	\end{aligned}
\end{equation}
where $G_{0 \to 1}^T = | \{ t \in \bar{\mathcal{G}} : g_t = \alpha, \bar{g}_t = \alpha - 1 \} |$ and $G_{1 \to 0}^T = | \{ t \in \bar{\mathcal{G}} : g_t = \alpha - 1, \bar{g}_t = \alpha \} |$ represent the flips of gradient from $\alpha$ to $\alpha - 1$ and from $\alpha - 1$ to $\alpha$, respectively.
The equality (a) holds because the difference $g_t - \bar{g}_t$ is $-1$ when $g_t \to \bar{g}_t$ is flipped from $\alpha - 1$ to $\alpha$, and $1$ when flipped from $\alpha$ to $\alpha - 1$.

Noting $g_t = \alpha - \mathbbm{1} \{ Y_t \notin \mathcal{C}_t(X_t) \}$, applying Lemma \ref{Lem_3_1} and taking absolute values, we have
\begin{equation}\label{Eq_App_2_4}
	\begin{aligned}
		\mathrm{MisCov}(T) = & \left| \alpha - \frac{1}{T} \sum_{t=1}^{T} \mathbbm{1}\{ Y_t \notin \mathcal{C}_t(X_t) \}  \right| \\
		\le & \frac{ B + \eta}{\eta T} + 
		\underbrace{ \frac{ | G_{0\to1}^T - G_{1\to0}^T | }{T} }_{\mathtt{term}_1}\\
		& + \underbrace{ \frac{\max\{ \alpha G_{1\to0}^T , (1 - \alpha) G_{0\to1}^T \}}{T} }_{\mathtt{term}_2} .
	\end{aligned}
\end{equation}
Here, $\mathtt{term}_1$ and $\mathtt{term}_2$ are defined for use in Appendix \ref{App_Coro_3_1}.

\section{Proof of Theorem \ref{Theo_3}}\label{App_Pro:wt_2}
\setcounter{equation}{0}
\renewcommand*{\theequation}{B-\arabic{equation}}

Similar to the proof of Theorem \ref{Theo_1}, we first show the boundedness of hidden variables $\{h_t\}$ under AC-ROCP.
\begin{lemma}\label{lem:2}
	Under Assumption \ref{Assumption_1}, with a bounded compensation term $w_t(q_t) \in [-W, W]$, the threshold $h_t$ produced by \eqref{Neq_31} satisfies, for all $t$,
	\begin{equation}
		-\eta ( W + 1 ) \le h_t \le B + \eta ( W + 1 ) .
	\end{equation} 
\end{lemma}
\begin{IEEEproof}
We define
\begin{equation}\label{eq_invariant}
	\mathcal{S} = \bigl[-\eta(\alpha + W),\ B + \eta(1 - \alpha + W)\bigr] .
\end{equation}
The base case $h_1 \in [0 , B] \subset \mathcal{S}$ is immediate. 
For the inductive step, assume $h_t\in\mathcal{S}$. If $t\in\mathcal{P}$, then
\eqref{Neq_31} gives $h_{t+1}=h_t$, and hence $h_{t+1}\in\mathcal{S}$.
It remains to consider the case $t\notin\mathcal{P}$. We divide it into three sub-cases.

\begin{itemize}
	\item $h_t \in (0, B)$. By the filter design, the update $h_{t+1} = h_t - \eta(\tilde{g}_t + w_t)$ applies with $\tilde{g}_t \in \{\alpha, \alpha - 1\}$ and $w_t \in [-W, W]$. Hence
	\begin{equation}
		h_{t+1} \in \bigl[h_t - \eta(\alpha + W),\ h_t + \eta(1 - \alpha + W)\bigr] \subset \mathcal{S}.\nonumber
	\end{equation}
	
	\item $h_t \in [B,\ B + \eta(1 - \alpha + W)]$. Since $h_t \ge B \ge s_t$, the ideal gradient is $g_t = \alpha$; the filter mechanism infers $g_t = \alpha$ and sets $w_t = 0$. Hence
	\begin{equation}
		h_{t+1} = h_t - \eta\alpha \in \bigl[B - \eta\alpha,\ B + \eta(1 - 2\alpha + W)\bigr] \subset \mathcal{S}. \nonumber
	\end{equation}
	
	\item $h_t \in [-\eta(\alpha + W),\ 0]$. Since $h_t \le 0 \le s_t$, the ideal gradient is $g_t = \alpha - 1$; the filter mechanism infers $g_t = \alpha - 1$ and sets $w_t = 0$. Hence
	\begin{equation}
		h_{t+1} = h_t + \eta(1 - \alpha) \in \bigl[\eta(1 - 2\alpha - W),\ \eta(1 - \alpha)\bigr] \subset \mathcal{S}.\nonumber
	\end{equation}
\end{itemize}
By noting $\max \{ \alpha , 1 - \alpha \} \le 1$, we finish this proof.
\end{IEEEproof}

We now prove Theorem~\ref{Theo_3}. 
The proof separates the contributions of training steps $t \in \mathcal{P}$ and non-training steps $t \in \bar{\mathcal{P}}$, and bounds the residual compensation error on the in-range subset $\mathcal{I} \subseteq \bar{\mathcal{P}}$.

\emph{Step 1: Training-step.} 
By the training protocol \eqref{eq_30}, each $t \in \mathcal{P}$ uses either $r_t = 0$ or $r_t = B$. 
This forces $e_t = 1$ when $r_t = 0$ and $e_t = 0$ when $r_t = B$. 
The protocol allocates $\lfloor \alpha|\mathcal{P}|\rfloor$ training steps to $r_t = 0$ and the remaining $\lceil (1-\alpha)|\mathcal{P}|\rceil$ to $r_t = B$, so
\begin{equation}
	\begin{aligned}
		& \sum_{t \in \mathcal{P}} \mathbbm{1} \{ Y_t \notin \mathcal{C}_t(X_t; r_t) \} \\
		& =
		\lceil (1 - \alpha) |\mathcal{P}| \rceil \times 0 + \lfloor \alpha |\mathcal{P}| \rfloor \times 1 \\
		& =
		\lfloor \alpha |\mathcal{P}| \rfloor ,
	\end{aligned}
\end{equation}
and therefore
\begin{equation}\label{Eq_App_Pro_5_1_2}
	\begin{aligned}
		\sum_{t \in \mathcal{P}} g_t 
		& = 
		\sum_{t \in \mathcal{P}} \left( \alpha - \mathbbm{1} \{ Y_t \notin \mathcal{C}_t(X_t; r_t) \} \right) \\
		& = 
		\alpha |\mathcal{P}| - \lfloor \alpha |\mathcal{P}| \rfloor \in [0 , 1).
	\end{aligned}
\end{equation}

\emph{Step 2: Non-training-step.} 
Unrolling \eqref{Neq_31} over the time index $t=1,\ldots,T$, and using $h_{t+1}=h_t$ for $t\in\mathcal{P}$, gives
\begin{equation}
	\begin{aligned}
		h_{T+1} = & h_{1} - \eta \sum_{t\in\bar{\mathcal{P}}} (\bar{g}_t + w_t(q_t)) \\
		\mathop = \limits^{\text{(a)}} & h_{1} - \eta \sum_{t\in\bar{\mathcal{P}}} {g}_t + \eta \sum_{t \in{\mathcal{I}}} \left( g_t - \bar{g}_t - w_t(q_t)\right) \\
		\mathop = \limits^{\text{(b)}} &  h_{1} - \eta \sum_{t=1}^T {g}_t + \eta \sum_{t \in \mathcal{I}} \left( g_t - \bar{g}_t  - w_t(q_t)\right) \\
		& + \eta (\alpha |\mathcal{P}| - \lfloor \alpha |\mathcal{P}| \rfloor) ,
	\end{aligned}
\end{equation}
where (a) adds and subtracts $\eta \sum_{t \in \bar{\mathcal{P}}} g_t$ and uses $w_t(q_t) = 0$ for all $ t \notin \mathcal{I}$, and (b) substitutes~\eqref{Eq_App_Pro_5_1_2} together with $\sum_{t \in \bar{\mathcal{P}}} g_t = \sum_{t=1}^T g_t - \sum_{t \in \mathcal{P}} g_t$.

\emph{Step 3: Deterministic bound.} 
Rearranging, taking absolute values and dividing by $\eta T$, we obtain
\begin{subequations}\label{Eq_App_Pro_5_1_4}
	\begin{align}
		\left| \frac{1}{T} \sum_{t=1}^T {g}_t \right| \le & \frac{\left|h_1 - h_{T+1}\right|}{\eta T} 
		+ \left|\frac{\alpha |\mathcal{P}| - \lfloor \alpha |\mathcal{P}| \rfloor}{T}\right| \nonumber \\
		& + \left|\frac{\sum_{t \in {\mathcal{I}}} \left( g_t - \bar{g}_t  - w_t(q_t) \right)}{T}\right| \\
		\mathop \le \limits^{(\text{a})} & \frac{B + \eta (W + 1)}{\eta T}  + \frac{1}{T} \nonumber \\
		&+ \left|\frac{\sum_{t \in {\mathcal{I}}} \left( g_t - \bar{g}_t  - w_t(q_t) \right)}{T}\right| \\
		\le & \frac{B + (W + 2)\eta}{\eta T} \nonumber \\
		& + \left|\frac{\sum_{t \in {\mathcal{I}}} \left( g_t - \bar{g}_t  - w_t(q_t) \right)}{T}\right| ,
	\end{align}
\end{subequations}
where (a) holds by Lemma \ref{lem:2} together with $\alpha |\mathcal{P}| - \lfloor \alpha |\mathcal{P}| \rfloor \le 1$.

\emph{Step 4: Compensation residual.} 
Using $g_t - \bar{g}_t = \bar{e}_t - e_t = (2\bar{e}_t - 1)z_t$ (from $g_t = \alpha - e_t$, $\bar{g}_t = \alpha - \bar{e}_t$), and the active compensation $w_t(q_t) = (2\bar{e}_t - 1) q_t$ from \eqref{eq_29}, we obtain
\begin{equation}
	\begin{aligned}
		\left|\frac{\sum_{t \in {\mathcal{I}}} \left( g_t - \bar{g}_t  - w_t(q_t) \right)}{T}\right|
		& = \frac{1}{T} \left| \sum_{t \in \mathcal{I}}  \left( 2 \bar{e}_t - 1 \right) \left(z_t - q_t\right)  \right| \\
		& \le \frac{1}{T} \sum_{t \in \mathcal{I}}  \left| 2 \bar{e}_t - 1  \right| \left| z_t - q_t \right| \\
		& \mathop \le \limits^{\text{(a)}} \frac{1}{T} \sum_{t \in \mathcal{I}} \left| z_t - q_t \right| ,
	\end{aligned}
\end{equation}
where (a) holds since $|2 \bar{e}_t - 1| = 1$.

\section{Proof of Corollary \ref{Coro_3_1}}\label{App_Coro_3_1}
\setcounter{equation}{0}
\renewcommand*{\theequation}{C-\arabic{equation}}

For $\mathtt{term}_1$, we recall that
\begin{eqnarray}
	G^T = G_{0\to1}^T +G_{1\to0}^T = \sum_{t=1}^T z_t.
\end{eqnarray}
Hence, $\mathtt{term}_1$ can be bounded by
\begin{equation}\label{Eq_App_3_1}
	\begin{aligned}
		\mathtt{term}_1
		& = 
		\frac{|G_{0\to1}^T-G_{1\to0}^T|}{T} \\
		& \le
		\frac{G_{0\to1}^T+G_{1\to0}^T}{T}
		=
		\frac{1}{T}\sum_{t=1}^T z_t.
	\end{aligned}
\end{equation}

For $\mathtt{term}_2$, since $G_{1\to0}^T\le G_{0\to1}^T+G_{1\to0}^T$ and $G_{0\to1}^T\le G_{0\to1}^T+G_{1\to0}^T$, we have
\begin{equation}
	\begin{aligned}
		& \max\{\alpha G_{1\to0}^T,(1-\alpha)G_{0\to1}^T\} \\
		& \qquad \qquad \le
		\max\{\alpha,1-\alpha\}(G_{0\to1}^T+G_{1\to0}^T).
	\end{aligned}
\end{equation}
Thus, it holds that
\begin{equation}\label{Eq_App_3_2}
	\begin{aligned}
		\mathtt{term}_2
		& =
		\frac{\max\{\alpha G_{1\to0}^T,(1-\alpha)G_{0\to1}^T\}}{T} \\
		& \le
		\max\{\alpha,1-\alpha\}\frac{1}{T}\sum_{t=1}^T z_t.
	\end{aligned}
\end{equation}

Combining \eqref{Eq_App_3_1} and \eqref{Eq_App_3_2} and Theorem \ref{Theo_1}, we have
\begin{equation}\label{Eq_App_3_3}
	\begin{aligned}
		& \left| \alpha - \frac{1}{T} \sum_{t=1}^{T} \mathbbm{1}\{ Y_t \notin \mathcal{C}_t(X_t) \}  \right|
		\le \frac{ B + \eta}{\eta T} \\
		& \qquad \qquad \qquad + \left(1 + \max\{ \alpha , 1-\alpha \}\right) \frac{1}{T}\sum_{t=1}^T z_t .
	\end{aligned}
\end{equation}

Under $z_t \overset{\mathrm{i.i.d.}}{\sim} \mathrm{Bern}(p)$, by directly applying Hoeffding's inequality, for any $\delta\in(0,1)$, with probability at least $1-\delta$, it holds that
\begin{equation}
	\frac{1}{T}\sum_{t=1}^T z_t
	\le
	p+\sqrt{\frac{\log(1/\delta)}{2T}}.
\end{equation}

Therefore, with probability at least $1-\delta$, it holds that
\begin{equation}
	\begin{aligned}
		& \left| \alpha - \frac{1}{T} \sum_{t=1}^{T} \mathbbm{1}\{ Y_t \notin \mathcal{C}_t(X_t) \}  \right|
		\le \frac{ B }{\eta T} + \frac{ 1 }{T} \\
		& \qquad + \left(1 + \max\{ \alpha , 1-\alpha \}\right) \left( p + \sqrt{\frac{\log(1/\delta)}{2T}} \right) .
	\end{aligned}
\end{equation}
By noting $1 + \max\{ \alpha , 1-\alpha \} \le 2$, we finish the proof.

\section{Proof of Corollary \ref{Coro_5_1}}\label{App_Coro_5_1}
\setcounter{equation}{0}
\renewcommand*{\theequation}{D-\arabic{equation}}

For notational convenience, we define $\mathcal{W}(x) = {x} / {(2 x - 1)}$ with $x \in [ 0 , \gamma]$, with $0 < \gamma < 0.5$.
We begin by decomposing the bound into two parts as follows
\begin{equation}
	\begin{aligned}
		& \frac{1}{T} \left| \sum_{t \in \mathcal{I}} \left( 2 \bar{e}_t - 1 \right) \left(z_t - q_t \right)  \right| \\
		= & \frac{1}{T} \left| \sum_{t \in \mathcal{I}} \left( 2 \bar{e}_t - 1 \right) \left(z_t - \mathcal{W}(\hat{p}) \right)  \right| \\
		= & \frac{1}{T} \left| \sum_{t \in \mathcal{I}} \left( 2 \bar{e}_t - 1 \right) \left(z_t - \mathcal{W}(p) + \mathcal{W}(p) - \mathcal{W}(\hat{p}) \right) \right| \\
		\le & \underbrace{ \frac{1}{T} \left| \sum_{t \in \mathcal{I}} \left( 2 \bar{e}_t - 1 \right) \left(z_t - \mathcal{W}(p) \right) \right| }_{\text{Part (a): Oracle Error}} \\
		& + \underbrace{ \frac{1}{T} \left| \sum_{t \in \mathcal{I}} \left( 2 \bar{e}_t - 1 \right) \left( \mathcal{W}(p) - \mathcal{W}(\hat{p}) \right) \right|  }_{\text{Part (b): Estimation Error}}   .
	\end{aligned}
\end{equation} 

Part (a) represents the oracle error, and Part (b) accounts for the estimation performance.
In the following, we derive the bounds for these two components separately.

\emph{Bounding Part (a):} 
Let $\mathcal{F}_{t-1} \!\!=\!\! \sigma(z_1 \!,\!\ldots \!,\! z_{t-1} \!,\!e_1 \!,\!\ldots \!,\!e_t)$. By construction, $e_t$ and the event $\{t \in \mathcal{I}\}$ are $\mathcal{F}_{t-1}$-measurable. Define
\begin{equation}
	N_t \;:=\; \mathbbm{1}\{t \in \mathcal{I}\} \cdot (2\bar{e}_t - 1)(z_t - \mathcal{W}(p)).
\end{equation}
Using $\bar{e}_t = e_t \oplus z_t$ and $z_t \sim \mathrm{Bern}(p)$ independent of $\mathcal{F}_{t-1}$, a direct calculation gives $\mathbb{E}[N_t \mid \mathcal{F}_{t-1}] = 0$, so $\{N_t\}$ is a martingale difference sequence. Moreover, $N_t$ takes values in $\{\mathcal{W}(p), 1 - \mathcal{W}(p)\}$ when $e_t = 0$ and in $\{-\mathcal{W}(p), \mathcal{W}(p) - 1\}$ when $e_t = 1$ (and is $0$ outside $\mathcal{I}$), so the range of $N_t$ is at most
\begin{equation}
	1 - 2\mathcal{W}(p) \;=\; \frac{1}{1 - 2p}.
\end{equation}
By the Azuma–Hoeffding inequality, with probability at least $1 - \delta_a$,
\begin{equation}
	\text{Part (a)} = \frac{1}{T}\left|\sum_{t \in \mathcal{I}} N_t\right| 
	\le \frac{1}{(1 - 2p) T}\sqrt{\frac{|\mathcal{I}| \log(2/\delta_a)}{2}}.
\end{equation}

\emph{Bounding Part (b):} Since $|2e_t-1| = 1$,  the triangle inequality gives 
\begin{equation}
	\begin{aligned}
		\text{Part (b)}
		& \le \frac{1}{T} \sum_{t \in \mathcal{I}} \left| \left( 2 \bar{e}_t - 1 \right) \right| \left| \left(\mathcal{W}(p) - \mathcal{W}(\hat{p}) \right) \right| \\
		& \le \frac{|\mathcal{I}|}{T} \left| \mathcal{W}(p) - \mathcal{W}(\hat{p}) \right| .
	\end{aligned}
\end{equation}
Since $\mathcal{W}'(x) = -1/(1 - 2x)^2$, the function $\mathcal{W}$ is $L$-Lipschitz on interval $[0, \max\{p, \hat{p}\}]$ with
\begin{equation}
	L = \frac{1}{(1 - 2\max\{p, \hat{p}\})^2} .
\end{equation}
Thus, Part (b) is further bounded by
\begin{equation}\label{Eq_App_7_1}
	\begin{aligned}
		\text{Part (b)} \le \frac{|\mathcal{I}| \left| p - \hat{p} \right|}{(1 - 2\max\{p , \hat{p}\})^2 T} .
	\end{aligned}
\end{equation}

It remains to bound term $| p - \hat{p} |$ in \eqref{Eq_App_7_1}, which measures the KT estimation error.
By plugging the KT estimator \eqref{eq_training_KT} into the term $| p - \hat{p} |$, we have
\begin{equation}
	\begin{aligned}
		| p - \hat{p} | 
		& =
		\left| p - \mathrm{Proj}_{[0, \gamma]}\left[ \frac{\sum_{t \in \mathcal{P}} z_t + 0.5}{|\mathcal{P}| + 1} \right] \right| \\
		& \le 
		\left| p - \frac{\sum_{t \in \mathcal{P}} z_t + 0.5}{|\mathcal{P}| + 1} \right| \\
		& \le
		\left| \frac{\sum_{t \in \mathcal{P}} z_t - |\mathcal{P}|p}{|\mathcal{P}| + 1}  \right| + \frac{1}{2(|\mathcal{P}| + 1)} .
	\end{aligned}
\end{equation}
Since $z_t$ follows i.i.d. $\mathrm{Bern}(p)$, Hoeffding's inequality gives, with probability at least $1 - \delta_b$,
\begin{equation}
	\left| \frac{\sum_{t \in \mathcal{P}} z_t - |\mathcal{P}|p}{|\mathcal{P}| + 1}  \right|
	\le
	\sqrt{\frac{\log(2/\delta_b)}{2 (|\mathcal{P}|  + 1)}} .
\end{equation}
Therefore, with probability at least $1-\delta_b$, it holds that
\begin{equation}
	| p - \hat{p} | 
	\le
	\sqrt{\frac{\log(2/\delta_b)}{2 (|\mathcal{P}|  + 1)}} + \frac{1}{2(|\mathcal{P}| + 1)}.
\end{equation}

\emph{Combining:} Setting $\delta_a = \delta_b = \delta/2$, we have that, with probability at least $1-\delta$, it holds that
\begin{equation}
	\begin{aligned}
		& \left| \alpha - \frac{1}{T} \sum_{t=1}^{T} \mathbbm{1}\{ Y_t \notin \mathcal{C}_t(X_t) \}  \right| \\ 
		&  \le \frac{B}{\eta T} + \frac{2}{T}
		+ \frac{1}{(1-2p)T} \sqrt{\frac{|\mathcal{I}| \log( 4/ \delta )}{2}}\\
		& \quad + \frac{|\mathcal{I}|}{(1-2\max\{p , \hat{p}\})^2 T} \left( \sqrt{\frac{\log(4/\delta)}{2(|\mathcal{P}| + 1)}} + \frac{1}{2(|\mathcal{P}| + 1)} \right) .
	\end{aligned}
\end{equation}

\section{Proof of Corollary \ref{Coro_3_2}}\label{App_Coro_3_2}	
\setcounter{equation}{0}
\renewcommand*{\theequation}{E-\arabic{equation}}

Based on the varriation of the corrupted indicator defined in \eqref{Neq_13}, we can partition the time horizon $T$ into $\lceil T / \Delta \rceil$ intervals, each of length at most $\Delta$.
In each interval, the corrupted indicator $z_t$ can only flip at most $f(\Delta)$ times, which implies that 
\begin{equation}
	\sum_{i=0}^{\Delta} z_{t+i} \le \max\{ f(\Delta) , \Delta - f(\Delta) \} = \hat{f}(\Delta).
\end{equation}

Hence, based on \eqref{Eq_App_3_3}, we have
\begin{equation}
	\begin{aligned}
		& \left| \alpha - \frac{1}{T} \sum_{t=1}^{T} \mathbbm{1}\{ Y_t \notin \mathcal{C}_t(X_t) \}  \right| \\
		& \le \frac{ B }{\eta T} + \frac{ 1 }{T} + 
		\left(1 + \max\{ \alpha , 1-\alpha \}\right) \frac{\lceil T / \Delta \rceil}{T} \hat{f}(\Delta) \\
		& \mathop \le \limits^{\text{(a)}} \frac{ B }{\eta T} + \frac{ 1 }{T} + 
		2 \hat{f}(\Delta) \left( \frac{1}{\Delta} + \frac{1}{T}\right) ,
	\end{aligned}
\end{equation}
where (a) holds since $\lceil T / \Delta \rceil \le T / \Delta + 1$ and $1 + \max\{ \alpha , 1-\alpha \} \le 2$.

\section{Proof of Corollary \ref{Coro_5_2}}\label{App_Coro_5_2}
\setcounter{equation}{0}
\renewcommand*{\theequation}{F-\arabic{equation}}

This proof will be based on \eqref{eq_28_b}.
Since we insert training point periodically, and hold out for subsequent $\Delta - 1$ steps, we have 
\begin{equation}
	\frac{1}{T} \sum_{t \in \mathcal{I}} \left| z_t - q_t \right|
	= \frac{1}{T} \sum_{i=0}^{\lceil T / \Delta \rceil - 1} \sum_{t = i \Delta + 1}^{\min\{ (i+1) \Delta , T \}} \left| z_t - q_t \right| .
\end{equation}

Recall that we assume the variation of the corruption indicator $z_t$ satisfies
\begin{equation}
	V(t , \Delta) = \sum_{i=0}^{\Delta} \mathbbm{1} \{ {z}_{t} \ne z_{t+i} \}
	\le f( \Delta ) .
\end{equation}

Thus, we have
\begin{equation}
	\begin{aligned}
		\frac{1}{T} \sum_{i=0}^{\lceil T / \Delta \rceil - 1} \sum_{t = i \Delta + 1}^{\min\{ (i+1) \Delta , T \}} \left| z_t - q_t \right|
		& \le
		\frac{1}{T} \sum_{i=0}^{\lceil T / \Delta \rceil - 1} f(\Delta) \\
		& \le
		\left( \frac{T}{\Delta} + 1 \right) \frac{f(\Delta)}{\Delta}.
	\end{aligned}
\end{equation}

Combining this with \eqref{eq_28_b} and noting $W = |(2 \bar{e}_t - 1) z_t| = 1$, we have
\begin{equation}
	\left| \alpha - \frac{1}{T} \sum_{t=1}^{T} \mathbbm{1}\{ Y_t \notin \mathcal{C}_t(X_t) \}  \right| 
	\le 
	\frac{B}{\eta T} + \frac{2}{T} + 
	\frac{f(\Delta)}{T} + \frac{f(\Delta)}{\Delta} .
\end{equation}

%% file: Main_ALL.bbl
\begin{thebibliography}{10}
\providecommand{\url}[1]{#1}
\csname url@samestyle\endcsname
\providecommand{\newblock}{\relax}
\providecommand{\bibinfo}[2]{#2}
\providecommand{\BIBentrySTDinterwordspacing}{\spaceskip=0pt\relax}
\providecommand{\BIBentryALTinterwordstretchfactor}{4}
\providecommand{\BIBentryALTinterwordspacing}{\spaceskip=\fontdimen2\font plus
\BIBentryALTinterwordstretchfactor\fontdimen3\font minus
  \fontdimen4\font\relax}
\providecommand{\BIBforeignlanguage}[2]{{%
\expandafter\ifx\csname l@#1\endcsname\relax
\typeout{** WARNING: IEEEtran.bst: No hyphenation pattern has been}%
\typeout{** loaded for the language `#1'. Using the pattern for}%
\typeout{** the default language instead.}%
\else
\language=\csname l@#1\endcsname
\fi
#2}}
\providecommand{\BIBdecl}{\relax}
\BIBdecl

\bibitem{gawlikowski2023survey}
J.~Gawlikowski \emph{et~al.}, ``A survey of uncertainty in deep neural
  networks,'' \emph{Artif. Intell. Rev.}, vol.~56, no. Suppl 1, pp. 1513--1589,
  2023.

\bibitem{abdar2021review}
M.~Abdar \emph{et~al.}, ``A review of uncertainty quantification in deep
  learning: Techniques, applications and challenges,'' \emph{Inf. Fusion},
  vol.~76, pp. 243--297, 2021.

\bibitem{Bingcong2019TSP}
B.~Li, T.~Chen, and G.~B. Giannakis, ``Secure mobile edge computing in {IoT}
  via collaborative online learning,'' \emph{{IEEE} Trans. Signal Process.},
  vol.~67, no.~23, pp. 5922--5935, 2019.

\bibitem{Chlaily2023TSP}
S.~Chlaily, D.~Ratha, P.~Lozou, and A.~Marinoni, ``On measures of uncertainty
  in classification,'' \emph{{IEEE} Trans. Signal Process.}, vol.~71, pp.
  3710--3725, 2023.

\bibitem{Nir2025TSP}
Y.~Dahan, G.~Revach, J.~Dunik, and N.~Shlezinger, ``Bayesian {KalmanNet}:
  Quantifying uncertainty in deep learning augmented {Kalman} filter,''
  \emph{{IEEE} Trans. Signal Process.}, vol.~73, pp. 2558--2573, 2025.

\bibitem{Dua2026ICASSP}
S.~Dua, G.~Mateos, and S.~P. Chepuri, ``Conformal inference for time series
  over graphs,'' in \emph{Proc. IEEE Int. Conf. Acoust., Speech, Signal
  Process.}, 2026, pp. 181--185.

\bibitem{simeone2025uncertainty}
O.~Simeone and Y.~Romano, ``Uncertainty-aware data-efficient {AI}: An
  information-theoretic perspective,'' \emph{arXiv preprint arXiv:2512.05267},
  2025.

\bibitem{shorinwa2025survey}
O.~Shorinwa, Z.~Mei, J.~Lidard, A.~Z. Ren, and A.~Majumdar, ``A survey on
  uncertainty quantification of large language models: Taxonomy, open research
  challenges, and future directions,'' \emph{ACM Comput. Surv.}, vol.~58,
  no.~3, pp. 1--38, 2025.

\bibitem{ankile2024juicer}
L.~Ankile, A.~Simeonov, I.~Shenfeld, and P.~Agrawal, ``Juicer: Data-efficient
  imitation learning for robotic assembly,'' in \emph{Proc. IEEE/RSJ Int. Conf.
  Intell. Robots Syst.}\hskip 1em plus 0.5em minus 0.4em\relax IEEE, 2024, pp.
  5096--5103.

\bibitem{Ji2024TSP}
Z.~Ji, C.~Chen, and X.~Guan, ``Observability guaranteed distributed intelligent
  sensing for industrial cyber-physical system,'' \emph{{IEEE} Trans. Signal
  Process.}, vol.~72, pp. 5198--5212, 2024.

\bibitem{zecchin2026prediction}
M.~Zecchin, U.~K. Ganesan, G.~Durisi, P.~Popovski, and O.~Simeone,
  ``Prediction-powered communication with distortion guarantees,'' \emph{IEEE
  J. Sel. Areas Inf. Theory}, 2026.

\bibitem{zhu2025conformal}
M.~Zhu, M.~Zecchin, S.~Park, C.~Guo, C.~Feng, P.~Popovski, and O.~Simeone,
  ``Conformal distributed remote inference in sensor networks under reliability
  and communication constraints,'' \emph{{IEEE} Trans. Signal Process.}, 2025.

\bibitem{simeone2025conformal}
O.~Simeone, S.~Park, and M.~Zecchin, ``Conformal calibration: Ensuring the
  reliability of black-box {AI} in wireless systems,'' \emph{{IEEE} Commun.
  Mag.}, pp. 1--9, 2025.

\bibitem{angelopoulos2020uncertainty}
A.~Angelopoulos, S.~Bates, J.~Malik, and M.~I. Jordan, ``Uncertainty sets for
  image classifiers using conformal prediction,'' \emph{arXiv preprint
  arXiv:2009.14193}, 2020.

\bibitem{chernozhukov2018exact}
V.~Chernozhukov, K.~W{\"u}thrich, and Z.~Yinchu, ``Exact and robust conformal
  inference methods for predictive machine learning with dependent data,'' in
  \emph{Proc. Conf. Learn. Theory}.\hskip 1em plus 0.5em minus 0.4em\relax
  PMLR, Jul. 2018, pp. 732--749.

\bibitem{gibbs2021adaptive}
I.~Gibbs and E.~Candes, ``Adaptive conformal inference under distribution
  shift,'' in \emph{Proc. Adv. Neural Inf. Process. Syst.}, vol.~34, Dec. 2021,
  pp. 1660--1672.

\bibitem{gibbs2024conformal}
I.~Gibbs and E.~J. Cand{\`e}s, ``Conformal inference for online prediction with
  arbitrary distribution shifts,'' \emph{J. Mach. Learn. Res.}, vol.~25, no.
  162, pp. 1--36, 2024.

\bibitem{bhatnagar2023improved}
A.~Bhatnagar, H.~Wang, C.~Xiong, and Y.~Bai, ``Improved online conformal
  prediction via strongly adaptive online learning,'' in \emph{Proc. Int. Conf.
  Mach. Learn.}\hskip 1em plus 0.5em minus 0.4em\relax PMLR, Jul. 2023, pp.
  2337--2363.

\bibitem{zecchin2024localized}
M.~Zecchin and O.~Simeone, ``Localized adaptive risk control,'' in \emph{Proc.
  Adv. Neural Inf. Process. Syst.}, vol.~37, Dec. 2024, pp. 8165--8192.

\bibitem{wu2025error}
J.~Wu, D.~Hu, Y.~Bao, S.-T. Xia, and C.~Zou, ``Error-quantified conformal
  inference for time series,'' in \emph{Int. Conf. Learn. Represent.}, Apr.
  2025.

\bibitem{xu2023conformal}
C.~Xu and Y.~Xie, ``Conformal prediction for time series,'' \emph{{IEEE} Trans.
  Pattern Anal. Mach. Intell.}, vol.~45, no.~10, pp. 11\,575--11\,587, 2023.

\bibitem{wang2025mirror}
B.~Wang, M.~Zecchin, and O.~Simeone, ``Mirror online conformal prediction with
  intermittent feedback,'' \emph{{IEEE} Signal Process. Lett.}, vol.~32, pp.
  2888--2892, 2025.

\bibitem{Fishel2025TSP}
E.~Fishel, M.~Malka, S.~Ginzach, and N.~Shlezinger, ``Remote inference over
  dynamic links via adaptive rate deep task-oriented vector quantization,''
  \emph{{IEEE} Trans. Signal Process.}, vol.~73, pp. 3557--3571, 2025.

\bibitem{alimohammadi2024kpi}
H.~Alimohammadi, S.~Chatzimiltis, S.~Mayhoub, M.~Shojafar, S.~A. Soleymani,
  A.~Akbas, and C.~H. Foh, ``{KPI} poisoning: An attack in open {RAN} near
  real-time control loop,'' in \emph{Proc. IEEE Future Netw. World
  Forum}.\hskip 1em plus 0.5em minus 0.4em\relax IEEE, 2024, pp. 712--718.

\bibitem{einbinder2024label}
B.-S. Einbinder, S.~Feldman, S.~Bates, A.~N. Angelopoulos, A.~Gendler, and
  Y.~Romano, ``Label noise robustness of conformal prediction,'' \emph{J. Mach.
  Learn. Res.}, vol.~25, no. 328, pp. 1--66, 2024.

\bibitem{feldman2025conformal}
S.~Feldman, S.~Bates, and Y.~Romano, ``Conformal prediction with corrupted
  labels: Uncertain imputation and robust re-weighting,'' \emph{arXiv preprint
  arXiv:2505.04733}, 2025.

\bibitem{xi2025exploring}
H.~Xi, K.~Liu, H.~Zeng, W.~Sun, and H.~Wei, ``Exploring the noise robustness of
  online conformal prediction,'' in \emph{Proc. Adv. Neural Inf. Process.
  Syst.}, Dec. 2025.

\bibitem{Chi2026ICASSP}
Y.-Q. Chi, C.-C. Zong, T.~Jin, and S.-J. Huang, ``Learning from noisy labels: A
  conformal prediction perspective,'' in \emph{Proc. IEEE Int. Conf. Acoust.,
  Speech, Signal Process.}, 2026, pp. 3776--3780.

\bibitem{zinkevich2003online}
M.~Zinkevich, ``Online convex programming and generalized infinitesimal
  gradient ascent,'' in \emph{Proc. Int. Conf. Mach. Learn.}, 2003, pp.
  928--936.

\bibitem{orabona2019modern}
F.~Orabona, ``A modern introduction to online learning,'' \emph{arXiv preprint
  arXiv:1912.13213}, 2019.

\bibitem{flaxman2004online}
A.~D. Flaxman, A.~T. Kalai, and H.~B. McMahan, ``Online convex optimization in
  the bandit setting: gradient descent without a gradient,'' \emph{arXiv
  preprint cs/0408007}, 2004.

\bibitem{zhang2022parameter}
J.~Zhang and A.~Cutkosky, ``Parameter-free regret in high probability with
  heavy tails,'' in \emph{Proc. Adv. Neural Inf. Process. Syst.}, vol.~35, Nov.
  2022, pp. 8000--8012.

\bibitem{van2019user}
D.~van~der Hoeven, ``User-specified local differential privacy in unconstrained
  adaptive online learning,'' in \emph{Proc. Adv. Neural Inf. Process. Syst.},
  vol.~32, Dec. 2019.

\bibitem{Cao2019TAC}
X.~Cao and K.~J.~R. Liu, ``Online convex optimization with time-varying
  constraints and bandit feedback,'' \emph{{IEEE} Trans. Autom. Control},
  vol.~64, no.~7, pp. 2665--2680, 2019.

\bibitem{van2021robust}
T.~van Erven, S.~Sachs, W.~M. Koolen, and W.~Kotlowski, ``Robust online convex
  optimization in the presence of outliers,'' in \emph{Proc. Conf. Learn.
  Theory}.\hskip 1em plus 0.5em minus 0.4em\relax PMLR, 2021, pp. 4174--4194.

\bibitem{zhang2025unconstrained}
J.~Zhang and A.~Cutkosky, ``Unconstrained robust online convex optimization,''
  \emph{arXiv preprint arXiv:2506.12781}, 2025.

\bibitem{angelopoulos2025gradient}
A.~N. Angelopoulos, M.~I. Jordan, and R.~J. Tibshirani, ``Gradient equilibrium
  in online learning: Theory and applications,'' \emph{arXiv preprint
  arXiv:2501.08330}, 2025.

\bibitem{angelopoulos2024online}
A.~N. Angelopoulos, R.~F. Barber, and S.~Bates, ``Online conformal prediction
  with decaying step sizes,'' \emph{arXiv preprint arXiv:2402.01139}, 2024.

\bibitem{hu2026distributioninformed}
D.~Hu, J.~Wu, S.-T. Xia, and C.~Zou, ``Distribution-informed online conformal
  prediction,'' in \emph{Int. Conf. Learn. Represent.}, Apr. 2026.

\bibitem{cover1999elements}
T.~M. Cover, \emph{Elements of information theory}.\hskip 1em plus 0.5em minus
  0.4em\relax John Wiley \& Sons, 1999.

\bibitem{KT1981TIT}
R.~Krichevsky and V.~Trofimov, ``The performance of universal encoding,''
  \emph{{IEEE} Trans. Inf. Theory}, vol.~27, no.~2, pp. 199--207, 1981.

\bibitem{krizhevsky2009learning}
\BIBentryALTinterwordspacing
A.~Krizhevsky and G.~Hinton, ``Learning multiple layers of features from tiny
  images,'' University of Toronto, Toronto, ON, Canada, Tech. Rep., Apr. 2009.
  [Online]. Available:
  \url{https://www.cs.toronto.edu/~kriz/learning-features-2009-TR.pdf}
\BIBentrySTDinterwordspacing

\bibitem{he2016deep}
K.~He, X.~Zhang, S.~Ren, and J.~Sun, ``Deep residual learning for image
  recognition,'' in \emph{Proc. IEEE Conf. Comput. Vis. Pattern Recognit.},
  Jun. 2016, pp. 770--778.

\bibitem{murray2012ava}
N.~Murray, L.~Marchesotti, and F.~Perronnin, ``{AVA}: A large-scale database
  for aesthetic visual analysis,'' in \emph{Proc. IEEE Conf. Comput. Vis.
  Pattern Recognit.}, Jun. 2012, pp. 2408--2415.

\bibitem{levin2017markov}
D.~A. Levin and Y.~Peres, \emph{Markov chains and mixing times}.\hskip 1em plus
  0.5em minus 0.4em\relax American Mathematical Soc., 2017, vol. 107.

\end{thebibliography}
